\documentclass{article}
\usepackage{makecell}
\usepackage{bbm}
\usepackage{epsfig}
\usepackage{latexsym}
\usepackage{setspace}
\usepackage{flushend}
\usepackage{authblk}
\usepackage{float}
\usepackage{dsfont}
\usepackage{comment}
\usepackage{ upgreek }
\usepackage{caption}
\usepackage{stmaryrd}
\usepackage{xspace}
\usepackage{url}
\usepackage{amsmath,amssymb,amsfonts,amsthm,color,xspace}
\usepackage{subfigure,graphicx,tikz,enumitem}

\usepackage[utf8]{inputenc}
\usepackage[utf8]{inputenc} 
\usepackage[T1]{fontenc}    
\usepackage{url}            
\usepackage{booktabs}       
\usepackage{amsfonts}       
\usepackage{nicefrac}       
\usepackage{microtype}      

\newcommand{\ba}{\begin{align}}
\newcommand{\ea}{\end{align}}
\newcommand{\ban}{\begin{align*}}
\newcommand{\ean}{\end{align*}}
\newcommand{\lbl}[1]{\label{}}

\newcommand{\br}{\begin{remark}}
\newcommand{\er}{\end{remark}}

\makeatletter
\newtheorem*{rep@theorem}{\rep@title}
\newcommand{\newreptheorem}[2]{%
\newenvironment{rep#1}[1]{%
 \def\rep@title{#2 \ref{##1}}%
 \begin{rep@theorem}}%
 {\end{rep@theorem}}}
\makeatother

\newtheorem{theorem}{Theorem}

\newtheorem{corollary}[theorem]{Corollary}
\newtheorem*{corollary*}{Corollary}

\newtheorem*{observation*}{Observation}

\newtheorem{lemma}[theorem]{Lemma}
\newtheorem*{lemma*}{Lemma}

\newreptheorem{theorem}{Theorem}
\newreptheorem{lemma}{Lemma}

\theoremstyle{remark}
\newtheorem{remark}{Remark}
\newtheorem*{remark*}{Remark}
\newtheorem*{remarks*}{Remarks}

\theoremstyle{definition}

\renewcommand{\qedsymbol}{\ensuremath{\blacksquare}}

\def\undertilde#1{\mathord{\vtop{\ialign{##\crcr
$\hfil\displaystyle{#1}\hfil$\crcr\noalign{\kern1.5pt\nointerlineskip}
$\hfil\tilde{}\hfil$\crcr\noalign{\kern1.5pt}}}}}

\allowdisplaybreaks
\newcommand{\E}{\mathds{E}}

\newcommand{\mtrobs}{{X}}
\newcommand{\mtrobsele}{{\mtrobs}_{i,j}}
\newcommand{\difmat}{{\mtrobs-\mtrmn}}

\newcommand{\ri}{r_i}
\newcommand{\cj}{c_j}
\newcommand{\eri}{\mathds{E}r_i}
\newcommand{\ecj}{\mathds{E}c_j}

\newcommand{\lambdaij}{\lambda_{i,j}}

\newcommand{\mtrestest}{M^{\text{est}}}

\newcommand{\gnrmtr}{A}
\newcommand{\mtrmn}{M}
\newcommand{\mtrmnij}{\mtrmn_{i,j}}

\newcommand{\mtrnew}{\widehat{M}}

\newcommand{\navg}{{n}_{\text{avg}}}

\newcommand{\noisemtralg}{N}
\newcommand{\noisemtrspt}{\noisemtralg}

\newcommand{\wcon}{W_{\scalebox{0.7}{\textnormal{cn}}}}

\newcommand{\wgen}{\tilde{w}_{\specpara}}
\newcommand{\wgenX}{\bar{w}_{\specpara}}

\newcommand{\nf}{D^{-\frac{1}{2}}(\nfx)}
\newcommand{\nb}{D^{-\frac{1}{2}}(\nbx)}

\newcommand{\nfem}{\nfm(i)}
\newcommand{\nbem}{\nbm(j)}

\newcommand{\nfex}{\nfx(i)}
\newcommand{\nbex}{\nbx(j)}

\newcommand{\nfm}{\wgen^\front}
\newcommand{\nbm}{\wgen^\back}

\newcommand{\nfx}{\wgenX^\front}
\newcommand{\nbx}{\wgenX^\back}

\newcommand{\swgen}{\tilde{w}}
\newcommand{\swgenX}{\bar{w}}

\newcommand{\snf}{D^{-\frac{1}{2}}(\snfx)}
\newcommand{\snb}{D^{-\frac{1}{2}}(\snbx)}

\newcommand{\snfex}{\snfx(i)}
\newcommand{\snbex}{\snbx(j)}

\newcommand{\snfx}{\swgenX^\front}
\newcommand{\snbx}{\swgenX^\back}

\newcommand{\front}{f}
\newcommand{\back}{b}

\newcommand{\elenoisemtrspt}{\noisemtrspt_{i,j}}
\newcommand{\specpara}{\Lambda}
\newcommand{\specparaN}{\specpara}

\newcommand{\gnrmtrotr}{B}
\newcommand{\gnrmtrotrtrSVD}{\gnrmtrotr^{\text{SVD}}_r}

\newcommand{\mtrrng}[3]{#1^{#2\times#3}}
\newcommand{\Rkm}{\mtrrng\reals km} 
\newcommand{\Rkk}{\mtrrng\reals kk}

\newcommand{\Rkkr}{\mtrrng{\reals_r}kk}

\newcommand{\Iell}{I}
\newcommand{\Iellbad}{\Iell_{\scalebox{0.7}\bad}}
\newcommand{\Ielldel}{\Iell_{\scalebox{0.7}\del}}
\newcommand{\Ielldelold}{\Iell_{\scalebox{0.7}\del}^{\textnormal{old}}} 
\newcommand{\Iellbd}{(\Iellbad \cup \Ielldel)}
\newcommand{\Iellstr}{\Iell_{\scalebox{0.7}{\textnormal{hv}}}}

\newcommand{\comp}{^{\scalebox{0.7}{\textnormal{c}}}}

\DeclareMathOperator{\textsum}{\textstyle{\sum}}

\newcommand{\norm}[2]{||#2||_{#1}}

\newcommand{\rank}{\textnormal{rank}}

\newcommand{\bad}{\textnormal{cn}}
\newcommand{\del}{\textnormal{zr}}

\newcommand{\Poi}{\mathrm{Poi}}
\newcommand{\Ber}{\mathrm{Ber}}
\newcommand{\Bin}{\mathrm{Bin}}

\newcommand{\whp}{w.h.p.\xspace}

\newtheorem{Theorem*}{Theorem}

\newtheorem{Claim*}[Theorem]{Claim}

\newtheorem{CounterExample*}{$\overline{\hbox{\bf Example}}$}

\newtheorem{Example*}[Theorem]{Example}

\newtheorem{Intuition*}[Theorem]{Intuition}
\newtheorem{Joke*}[Theorem]{Joke}

\newtheorem{Lemma*}[Theorem]{Lemma}
\newtheorem{Open problem}[Theorem]{Open problem}

\newtheorem{Question*}[Theorem]{Question}

\newcommand{\ignore}[1]{}

\newcommand{\RR}{\mathbb{R}}

\def \ba     {{\bf a}}

\newcommand{\cM}{{\mathcal M}}

\newcommand{\cO}{{\mathcal O}}

\newcommand{\reals}{\RR}
\newcommand{\eg}{\textit{e.g.,}\xspace}

\definecolor{light}{gray}{.75}

\def \upto  {{,}\ldots{,}}

\def \paren#1{{({#1})}}
\def \Paren#1{{\left({#1}\right)}}

\def\ignore#1{}

\newcommand{\bi}{\begin{itemize}}
\newcommand{\ei}{\end{itemize}}

\def\orpro{\mathop{\mathchoice
   {\vee\kern-.49em\raise.7ex\hbox{$\cdot$}\kern.4em}
   {\vee\kern-.45em\raise.63ex\hbox{$\cdot$}\kern.2em}
   {\vee\kern-.4em\raise.3ex\hbox{$\cdot$}\kern.1em}
   {\vee\kern-.35em\raise2.2ex\hbox{$\cdot$}\kern.1em}}\limits}

\def\andpro{\mathop{\mathchoice
 {\wedge\kern-.46em\lower.69ex\hbox{$\cdot$}\kern.3em}
 {\wedge\kern-.46em\lower.58ex\hbox{$\cdot$}\kern.25em}
 {\wedge\kern-.38em\lower.5ex\hbox{$\cdot$}\kern.1em}
 {\wedge\kern-.3em\lower.5ex\hbox{$\cdot$}\kern.1em}}\limits}

\def\simge{\mathrel{%
   \rlap{\raise 0.511ex \hbox{$>$}}{\lower 0.511ex \hbox{$\sim$}}}}

\def\simle{\mathrel{
   \rlap{\raise 0.511ex \hbox{$<$}}{\lower 0.511ex \hbox{$\sim$}}}}

\usepackage{microtype}
\usepackage{graphicx}
\usepackage{subfigure}
\usepackage{booktabs} 

\usepackage{hyperref}

\hypersetup{colorlinks,linkcolor={blue},citecolor={blue},urlcolor={red}} 

\usepackage[preprint,nonatbib]{neurips_2020}
\usepackage{selectp}

\usepackage{algorithm}
\usepackage{algorithmic}

\setlist[enumerate]{leftmargin=5.5mm}

\title{Linear-Sample Learning of Low-Rank Distributions}

\author{%
  Ayush Jain and Alon Orlitsky\\
  Dept. of Electrical and Computer Engineering\\
  University of California, San Diego\\
  \texttt{ayjain@eng.ucsd.edu, alon@ucsd.edu} 
}

\begin{document}
\maketitle
\begin{abstract}

Many latent-variable applications, including community detection, collaborative filtering, genomic analysis, and NLP, model data as generated by low-rank matrices. Yet despite considerable research, except for very special cases, the number of samples required to efficiently recover the underlying matrices has not been known. 

We determine the onset of learning in several common latent-variable settings. For all of them, we show that learning $k\hskip-.09em\times\hskip-.09em k$, rank-$r$, matrices to 
normalized $L_{\hskip-.02em1}$
distance~$\epsilon$ requires  $\Omega(\frac{kr}{\epsilon^2})$ samples, and propose an algorithm that uses ${\cal O}(\frac{kr}{\epsilon^2}\log^2\frac r\epsilon)$ samples, a number linear in the high dimension, and
nearly
linear in the, typically low,~rank.

The algorithm improves on existing
spectral techniques
and
runs in polynomial time.
The proofs establish 
new results on the rapid convergence of the spectral distance between the model and observation matrices, and may be of independent interest.\looseness-1

\end{abstract}

\section{Introduction}
\subsection{Motivation}
A great many
scientific and technological applications concern relations between two objects that range over large domains, yet are linked via a low-dimensional \emph{latent} space. Often this relation is precisely, or nearly, linear, hence can be modeled by a low-rank matrix. 
The problem of recovering low-rank model matrices from observations they generate has therefore been studied extensively. 
Following are five of the most common settings considered, each introduced via a typical application.

\textbf{Distribution matrices\quad}
In \emph{Probabilistic latent semantic analysis} samples are co-occurrences $(w,d)$ of words and documents, assumed independent given one of $r$ latent topic classes $t$~\cite{hofmann1999probabilistic}.
The joint probability matrix $[p_{w,d}]$ is therefore a mixture of at most $r$ product matrices $p(d|t)\cdot p(w|t)$, and has rank $\le r$. This setting also arises in many hidden Markov applications, e.g.,~\cite{mossel2005learning,hsu2012spectral}.

For this setting, the model matrix is a distribution, hence its elements are non-negative and sum to 1. The model is sampled independently $n$ times, and  $X_{i,j}$ is the number of times pair $(i,j)$ was observed.
In other settings, the model matrix consists of arbitrary parameters, it is sampled just once, and each parameter is reflected in an independent observation $X_{i,j}$.

\textbf{Poisson Parameters}\quad
Recommendation systems infer consumer preferences from their consumption patterns. The number of times customer $i$ purchases product $j$ is typically modeled as an independent Poisson random variable $X_{i,j}\!\sim \Poi(\lambdaij)$. Often $\lambda_{i,j}$ is the inner product the consumer's disposition and 
product's expression of $r$ latent features~\cite{sarwar2000application}, 
so the parameter matrix $[\lambda_{i,j}]$ has rank $\le r$.\looseness-1

\textbf{Bernoulli Parameters}\quad
In \emph{inhomogeneous Erd\"os-R\'enyi graphs}, the edge between nodes $i$ and $j$ is an independent random variable $X_{i,j}\sim\Ber(p_{i,j})$.
In \emph{Community-detection}, the Stochastic Block Model (SBM), \eg~\cite{abbe2017community,mossel2018proof,abbe2015exact,bordenave2015non}, assumes that graph nodes fall in few communities $C_1\upto C_r$ such that the $p_{i,j}$'s are one constant if $i$ and $j$ are in the same community, and a different constant if $i$ and $j$ are in two distinct communities. Clearly, the parameter matrix $[p_{i,j}]$ has rank $\le r$.

\textbf{Binomial Parameters\quad} 
The probability that gene pair $(i,j)$ will express as a phenotype is often viewed the result of a few factors, resulting in an expression probability matrix $[p_{i,j}]$ of low rank at most $r$~\cite{kapur2016gene}. 
In a study of $t$ phenotypic patients, the number $X_{i,j}$ of patients with gene pair $(i,j)$ will therefore be distributed binomially $\Bin(t,p_{i,j})$.

\textbf{Collaborative filtering\quad} 
Let $Y_{i,j}\in [0,1]$ be the random rating user $i$ assigns to movie $j$. 
Assuming a small number $r$ of intrinsic film features, the matrix $[F_{i,j}]$,
where $F_{i,j}=\E[Y_{i,j}]$,
would have rank at most $r$. Since only some ratings are reported, the matrix-completion and collaborative-filtering literature, e.g.,~\cite{borgs2017thy}, assumes that $\mathbbm{1}_{i,j}\sim\Ber(p)$ are independent indicator random variables, and upon observing $X_{i,j} = \mathbbm{1}_{i,j}\cdot Y_{i,j}$ for all $(i,j)$
we wish to recover the mean matrix $[F_{i,j}]$.  

Each of these five settings has been studied in many additional contexts, including 
{word embedding~\cite{stratos2015model}},
{Genomic Analysis~\cite{zhou2014piecewise,zou2013contrastive}},
and more~\cite{anandkumar2014tensor}.

\subsection{Unified formulation} 
We first unify the five settings, facilitating the interpretation of their matrix norm as the number of observations, and allowing us to subsequently show that all models have essentially the same answer. 

Both our lower and upper bounds for recovering $k\times m$ matrices depend only on the larger of $k$ and $m$. Hence without loss of generality we assume square model matrices. Let $\Rkk$ be the collection of $k\times k$ real matrices, and let $\Rkkr$ be its subset of matrices with rank at most $r$.

Note that in the first four models, $X_{i,j}$ is the number of times pair $(i,j)$ was observed. For example, for the Poisson model, it is the number of times customer $i$ bought product $j$. Hence
$||X||_1=\sum_{i,j} X_{i,j}$ is the total number of observations. 
In collaborative filtering, 
$X_{i,j} = \mathbbm{1}_{i,j}\cdot Y_{i,j}$ does not carry the same interpretation, still $||X||_1=\sum_{i,j} X_{i,j}$ will concentrate around $p\cdot k^2\cdot F_{i,j}^{\text{avg}}$, hence will be proportional to the expected number of observations, $pk^2$. 

We therefore scale the matrix $M$ so that $M_{i,j}=E(X_{i,j})$ reflects the expected number of observations of  $X_{i,j}$. 
We let $M$ be $n\cdot[p_{i,j}]$
for distribution matrices, $[\lambda_{i,j}]$ for Poisson parameters,
$[p_{i,j}]$ for Bernoulli parameters,  $[t\cdot p_{i,j}]$ for Binomial parameters, and $[p \cdot F_{i,j}]$ for collaborative filtering. 
Consequently, for all models $M=\E[X]$,
hence 
$
||M||_1=||\E[X]||_1
$ is the expected number of
observations.\looseness-1

Let $M$ be an unknown model matrix for one of the five settings, and let $X\sim M$ be a resulting observation matrix. We would like to recover $M$ from $X$. 
A model estimator is a mapping $\mtrestest:\Rkk\to\Rkk$ that associates with an observation $X$ an estimated model matrix $\mtrestest:=\mtrestest(X)$. 

Different communities have used different 
measures for how well $\mtrestest$ approximates $M$.
For example community detection concerns the recovery of labels, a criterion quite specific to this particular application.
Many other works considered recovery in squared-error, or Forbenius, norm, but as 
argued in Section~\ref{sub:relatedwork}, this measure is less meaningful for the applications we consider. 

Perhaps the most apt estimation measure, and the one we adopt, is $L_1$ distance, 
the standard Machine-Learning accuracy measure. $L_1$ arises naturally in numerous applications and is the main criterion used for learning distributions.
Since $M$ has non-unitary norm, we normalize the $L_1$ distance between $M$ and its estimate $\mtrestest$,
\[
L(\mtrestest) := 
L_M(\mtrestest) := 
||\mtrestest-M||_1/||M||_1.
\]
Note that for distribution matrices, this reduces to the standard $L_1$ norm.
Also, similar to total variation distance, $L(\mtrestest)$ upper bounds the absolute difference between the expected and predicted number of occurrences of pairs $(i,j)$ in any subset $S\subseteq [k]\times[k]$, normalized by the total observations.\looseness-1

While the multitude of applications has drawn considerable amount of work on these latent variable model, a lot more work assumes that the number of samples are plenty, way beyond the information theoretic limit, or assumes more stronger assumptions on the model matrix then just low rank, and limiting its applicability. We show recovery is possible with only linear number of samples in all these models and with no other assumptions on low rank matrix $M$.

\subsection{Overview of the background, main results and techniques}
Recovery of distribution-matrices (the first model) in $L_1$ distance was first addressed in~\cite{huang2018recovering}.
They considered matrices $M$ that can be factored as 
$UWU^\intercal$ where 
$W$ is $r\times r$ and positive semi-definite, and 
$U$ is non-negative.
They derived a polynomial-time algorithm that requires
$
\cO(w_{\!{}_M}kr^2/\epsilon^5)$ samples, where $w_{\!{}_M}\ge r^2$, hence at best guaranteeing sample complexity $\cO(kr^4/\epsilon^5)$, and potentially much higher.
Note however that their definition applies only to positive semi-definite, and not general matrices $M$, and even then, $r$ is at least the rank of $M$,
and can be significantly higher. 

We first lower bound the accuracy of any estimator. 
An array of $kr$ elements can be viewed as a special case of a $k\times k$ matrix of rank $r$.
Simply place the array's $kr$ elements in the matrix's first $r$ rows, and set the remaining rows to zero.
A well-known lower bound for learning discrete distributions, or arrays of Poisson parameters, in $L_1$ distance~\cite{kamath2015learning,han2015minimax} therefore implies:

\begin{theorem}\label{th:lowerb}
For any $k$, $r$, $\epsilon<1$, and $M\in \Rkkr$, let $X\sim M$ via the Poisson parameters or distribution matrix model. Then for any, possibly random, estimator $\mtrestest$, 
\[
\sup_{M\in \Rkkr:||M||_1\le kr/\epsilon^2 }\E_{X\sim M}[L(\mtrestest(X))]= \Omega(\epsilon).
\]
\end{theorem}
The bound implies that achieving expected normalized $L_1$ error $<\epsilon$ requires expected number of observations $||M||_1=\Omega(kr/\epsilon^2)$. Equivalently, any estimator incurs an expected normalized $L_1$ error at least $\Omega(\min\{\sqrt{kr/||M||_1},1\})$.

Our main result is a polynomial-time algorithm \emph{curated SVD} that returns an estimate $M^{\text{cur}}:=M^{\text{cur}}(X)$ that essentially achieves the lower bound for all five models and all matrices $M$. 

\begin{theorem} \label{th:main2}
Curated SVD runs in polynomial time, 
and for every $k$, $r$, $\epsilon>0$, and $M\in \Rkkr$ with $||M||_1\ge \frac{kr}{\epsilon^2}\cdot \log^2\frac{r}{\epsilon}$, if $X\sim M$, then with probability $\ge 1-k^{-2}$, 
\[
L(M^{\text{cur}}(X)) = \cO(\epsilon).
\]
\end{theorem}

A few observations are in order. 
While~\cite{huang2018recovering} provides weaker guarantees and only for a special subclass of matrices, Curated SVD achieves essentially the lower bound for all matrices. 
%
It also holds for all five models. 
It recovers $M$ with $\cO(kr\log^2 r)$ observations. This number is linear in the large matrix dimension $k$, and near linear in the typically small rank $r$. This is the first such result for general matrices. 

In many applications, only a small number of observations is available per row and column. For example on average each viewer may rate only few movies, and a person typically has few friends. Hence the number of observations is near linear in the dimension $k$. Our results are the first to enable general learning these regimes.

With $n$ samples, general discrete distributions
over $k$ elements can be learned to $L_1$ distance $\Theta(\sqrt{k/n})$. 
Theorems~\ref{th:lowerb} and~\ref{th:main2} show that essentially the same result holds for $L_1$ learning of low-rank matrices. The number of parameters is $kr$, the number of observations is $||M||_1$, and the normalized $L_1$ error is between $\sqrt{kr/||M||_1}$ and 
$\sqrt{kr/||M||_1}\log(r||M||_1/k)$.

To obtain these results we generalize a recent work~\cite{le2017concentration} that bounds the spectral distance between $M$ and $X$. This bound requires the strong condition that each entry of the corresponding Bernoulli parameter matrix $M$ should be within a constant factor from the average. The paper asked whether such results can be achieved for more general sparse graphs, possibly with the aid of regularization. We provide a counter example showing that such strong guarantees cannot hold for sparse graphs.

Instead, we derive a new spectral result (Theorem~\ref{th:sp norm}) that helps recover $M$ even when few observations are available, and may be of independent interest.
The result applies to all models, not just Bernoulli, and shows that even when the number of observations is only linear in $k$, zeroing out a small submatrix in $M$ and $X$, and then regularizing the two matrices, results in a small spectral distance between them. We believe this new result could potentially imply learning in the sparse sample regime for other settings that are not explicitly considered here.

\textbf{Theorem~\ref{th:sp norm} (Informal, full version in Section~\ref{sec:specnorm})} 
There is a small unknown set of rows that when zeroed out from regularized versions of $X$ and $M$ results in small spectral distance between them. 

Although this set of rows is unknown, we derive an algorithm that recovers $M$ to a small $L_1$ distance.

\textbf{Curated SVD (Informal, full version in Section~\ref{sec:recoveralg})} 
Successively zeroes out few suspicious rows, ensuring that any small set of rows of $X$ don't have too large an influence on the recovered matrix.

We show that Theorem~\ref{th:sp norm} implies that curated SVD achieves the recovery guaranties in Theorem~\ref{th:main2}.

\subsection{Implications of the results and related work}\label{sub:relatedwork}
As mentioned in the introduction, many communities have considered recovering model matrices from samples. Here we describe a few of the results and their relation to those presented in this paper. 

Community detection infers communal structure from few pairwise interactions between individuals. 
Much work has focused on the Stochastic Block models (SBM) where 
individuals fall in few communities $C_1\upto C_r$ and the interaction probabilities $p_{i,j}$ are one constant if $i$ and $j$ are in the same community, and a different constant if $i$ and $j$ are in two distinct communities.
Precise guarantees were provided for both exact recovery and detection~\cite{abbe2017community,abbe2015community,abbe2015exact,bordenave2015non}. 

Mixed membership SBM associates each individual with an unknown $r$-dimensional vector reflecting weighted membership in each of $r$ communities, and the pairwise interaction probability is determined by inner product of the respective membership vectors.
The resulting interaction probability matrix is a rank-$r$ Bernoulli parameter matrix. 
Recently,~\cite{hopkins2017bayesian} considered a Bayesian setting where the resulting membership weights are both sparse and evenly distributed,
and achieved weak detection with only $\cO(k r^2)$ samples.

Note that all these settings imply 
special model matrices. 
Even for the more general mixed-membership model in~\cite{hopkins2017bayesian} the assumptions imply that all entries in the resulting model matrix are within a constant factor from their average.
By contrast our analysis applies to all low rank matrices $M$, and for $L_1$ recovery using $\cO(kr\log^2 r)$ samples. We note that the goal in community detection setting is somewhat more specific than ours,
and the two guarantees may not be directly comparable.

Another recent work~\cite{mcrae2019low} considered recovering Poisson and distribution matrices under the Frobenius norm.
They showed that matrices $M$ with moderately-sized entries, can be recovered with $O(k\log^{3/2} k)$ samples, but the error depended on  $M$'s maximum row and column sum.
However, they note that the Frobenius norm "might not always be the
most appropriate error metric" and point out that 
the $L_1$ norm is "much stronger" for these settings. A similar sentiment about $L_1$ is echoed by~\cite{huang2018recovering} in relation to the spectral norms. 
By contrast, our results apply to the $L_1$ norm, grow only linearly with $k$, and the error depends on $M$'s average, not the maximum, row and column sum.

Our collaborative filtering setting uses the same general bounded noise model as~\cite{borgs2017thy}.
However they assume that the mean matrix $F$ is generated by a Lipschitz latent variable model. 
They recover $F$ to mean square error $\sum_{i,j}(F_{i,j}-\hat F_{i,j})^2/k^2 = \cO(r^2/(pk)^{2/5})$, implying that recovering $F$ requires $pk^2= \cO(r^5 k)$ samples.
They also provides a nice survey of related matrix completion and show that their result is the first to achieve linear in $k$ recovery for the general bounded noise model.

By contrast, we make no additional assumptions, and recover $F$ to $L_1$ distance with 
$\cO(k r\log^2 r)$ samples. 
Note also that $0\le F_{i,j}\le 1$, hence $|F_{i,j}-\hat F_{i,j}|\ge |F_{i,j}-\hat F_{i,j}|^2$.
Therefore normalized $L_1$ error upper bounds mean squared error. 
In Appendix~\ref{subsec:main2} we show that our estimator achieves a better error bound, $\sum_{i,j}|F_{i,j}-\hat F_{i,j}|/k^2 = \cO(\sqrt{r/pk}\cdot \log(rpk))$, even for the stronger $L_1$ norm.\looseness-1

Learning latent variables models 
has been addressed in several other communities that typically focused on computational efficiency when the data is in abundance, includes work on Topic Modelling~\cite{arora2013practical,ke2017new,bing2018fast}, and Hidden Markov Models~\cite{hsu2012spectral,mossel2005learning,anandkumar2014tensor}, word embedding~\cite{stratos2015model}, and Gaussian mixture models~\cite{dasgupta1999learning,ge2015learning,vempala2004spectral}.
\subsection{Arrangement of the paper}
The reminder of the paper is organised as follows. Section~\ref{sec:prelim} defines some useful notations and recalls some useful properties for these models.
Section~\ref{sec:specnorm} defines the regularization and establish bounds on the the regularized spectral distance between $X$ and $M$. Section~\ref{sec:recoveralg} describes the Curated-SVD algorithm to recover $M$ and gives an overview of its analysis.

\section{Preliminaries}\label{sec:prelim}

\subsection{Notation}
We will use the following formulation of rank, singular values, and decompositions. 
Let $\Rkm$ be the collection of $k$ by $m$ real matrices. Every matrix $\gnrmtr\in\Rkm$ can be expressed in terms of its \emph{singular-value decomposition (SVD)},
$
\gnrmtr
=
\textstyle{\sum}_{i=1}^{\min\{k,m\}} \sigma_i u_iv_i^\intercal ,
$
where the \emph{singular values} $\sigma_i:=\sigma_i(\gnrmtr)$ are non-increasing and non-negative, and the right singular vectors $v_i:= v_i(\gnrmtr )$ are orthogonal, as are the left singular vectors $u_i:= u_i(\gnrmtr)$.
A matrix has rank $r$ iff its first $r$ singular values are positive, and the rest are zero.
A \emph{$t$-truncated} SVD of a matrix is one where only the first $t$ singular values in the SVD are retained and the rest are discarded. 
It is easy to see that its rank is $\min(t,r)$.
For $t\le \min\{k,m\}$, let $A^{(t)}$ denote the $t$-truncated SVD of a matrix $\gnrmtr\in\Rkm$.

The \emph{$L_1$ "entry-wise" norm}, or \emph{$L_1$ norm}, of a matrix $\gnrmtr$ is 
$\norm1\gnrmtr := \textstyle{\sum}_{i,j}|\gnrmtr_{ij}|.$
Let $||A_{i,*}||_1:= \sum_j |A_{i,j}|$ and $||A_{*,j}||_1:= \sum_i |A_{i,j}|$ denote the $L_1$ norm of $i^{th}$ row and $j^{th}$ column of $A$, respectively.\looseness-1

The \emph{$L_2$ norm} of a vector $v=(v(1)\upto v(m))\in\reals^m$ is $\norm{}v:=\sqrt{\sum_{i=1}^m v(i)^2}$. The \emph{spectral norm}, or \emph{norm} for short, of a matrix $\gnrmtr\in\Rkm$ is 
$ 
||\gnrmtr|| := \max_{v\in \reals^m: ||v||=1}||\gnrmtr v|| = \max_{u\in \reals^k: ||u||=1}||\gnrmtr ^\intercal u||.
$


\subsection{A unified framework}
We first describe a unified common framework for all five problems. 

To unify distribution-matrices with Poisson-parameter matrices, we apply the  well-known \emph{Poisson trick}~\cite{szpankowski2001average}, where
instead of a fixed sample size $n$, we take $\Poi(n)$ samples. 
The resulting random variables $X_{i,j}$ will be Poisson and independent. 
Furthermore, since with probability $\ge 1-1/n^3$, the difference between $n$ and $\Poi(n)$ is $O(\sqrt{n\log n})$, this modification contributes only smaller order terms, and the algorithm and guarantees for Poisson parameters carry over to distribution matrices. 


Having unified distribution- and Poisson-parameter- matrices,
we focus on the remaining four settings.
In all of them, the observations 
$X_{i,j}$ are independent non-negative random variables with $\E[X_{i,j}] = M_{i,j}$ and  $\text{Var}( X_{i,j}) \le M_{i,j}$.
The last inequality clearly holds for the Poisson, Bernoulli, and Binomial matrices. For collaborative filtering it follows as
\[
{\text{Var}(X_{i,j})}
\le
E((X_{i,j})^2)
\le 
E(X_{i,j})=M_{i,j}.
\]

Define the \emph{noise matrix} $\noisemtrspt := X- M$ as the difference between $X$ and its expectation $M$. Note that the spectral distance between $X$ and $M$ is the same as the the spectral norm of the noise matrix $\noisemtrspt$.

Let $\navg := ||M||_1/k$ denote the average 
expected number of observations in each row and column. For simplicity, we assume $||M||_1$, the total number of expected observation, is known. 
Otherwise, since $\E[||X||_1] = ||M||_1$, it can be estimated very accurately. 

\section{Spectral norm of the regularized noise matrix}\label{sec:specnorm}

Recall that the noise matrix $N=X-M$, and its spectral norm $||N||$ is the spectral distance between observation matrix $X$ and $M$.
When $||N||=\cO(\sqrt{\navg})$, a simple truncated-SVD of $X$ can be shown to recover $M$.
Unfortunately, when some row- or column-sums of $M$ far exceed the average, $||N||$ can be quite large.
This is because the expected squared norm of row $i$ of $N$ is $\sum_j\text{Var}(X_{i,j})$ which is roughly  $||M_{i,*}||_1$, and when the row sum of matrix $M$ is non-uniform $||M_{i,*}||_1$ can be much larger than the average value $\navg$ across the rows.  Similarly for the columns.

The difficulty caused by the heavy rows and columns of $M$ can be mitigated by regularization that reduces their weight.
Let $w = (w^\front, w^\back)$
where $w^\front=((w^\front(1),..,w^\front(k))$ and  $w^\back=(w^\back(1),..,w^\back(k))$ are the row and columns regularization weights, all at least 1.
And let $D(u)$ be the diagonal matrix with entries $u(i)$.
The \emph{$w$-regularized} $A\in\Rkk$ is
\[
R(A, w) := D^{-\frac{1}{2}}(w^\front)\cdot A \cdot D^{-\frac12}(w^\back).
\]
Upon multiplying the $i^{th}$ row of matrix $A$ by $(w^\front(i))^{-1/2}$, its expected squared norm reduces by a factor $1/w^\front(i)$, and similarly for the columns.
Therefore, selecting regularization weights $\swgen = (\swgen^\front,\swgen^\back)$, where
$\swgen^\front(i)
=
\textstyle\max\{1, ||M_{i,*}||_1/\navg\} \text{ and }\swgen^\back(j) = \max\{1,||M_{*,j}||_1/\navg\}$ would 
reduce the expected squared norm of heavy rows and columns of $R(N,\swgen)$ to $\navg$. 

Unfortunately, the $||M_{i,*}||_1$ and $||M_{*,j}||_1$'s are not known, hence neither is $\swgen$. However,  $\E[\sum_{j}X_{i,j}] = ||M_{i,*}||_1$, hence we approximate $\swgen$ by $\swgenX = (\snfx,\snbx)$, where
\[
\snfex
:=
\textstyle\max\{1, \frac{||X_{i,*}||_1}{\navg}\}
\text{\quad and\quad }
\snbex
:=
\max\{1,\frac{||X_{*,j}||_1}{\navg}\}.
\]

Unless specified otherwise we use weights $\swgenX$, and refer to them as \emph{weights}.
This is one of the several commonly used regularizations in 
spectral methods 
for community detection~\cite{chaudhuri2012spectral,qin2013regularized,joseph2013impact}.

When $\navg = o(\log k)$, some rows and columns may have regularization weights that are below the ideal weights, ${\snfex}\ll  {\swgen^\front(i)}$ and ${\snbex}\ll {\swgen^\back(j)}$, 
and will not be regularized properly.
Yet as shown in Theorem~\ref{th:sp norm}, our technique can handle these problematic rows and columns as well.

The \emph{regularized spectral distance} between two matrices $A$ and $B$ is 
$||R(A-B, w)||$, the spectral norm of their regularized spectral difference. 
Lemma~\ref{th:recovery} in the next section relates $|| R(N,\swgenX)||$ to $L_1$ recovery guarantees, and implies that when 
$|| R(N,\swgenX)||\le \cO(\sqrt{\navg})$ a simple variation of truncated SVD can recover $M$ from $X$ to the minmax lower bound on $L_1$ recovery in Theorem~\ref{th:lowerb}. 

When the number of samples is at least a few $\log k$ factors more than $k$, this spectral concentration could probably be achieved with the help of the regularization. 
The more interesting, challenging, and prevalent setting, is when few observations are available, and this is the main focus of the paper.

For Bernoulli Parameter matrices, \cite{le2017concentration} obtained the tight bound $||R(N,\swgenX) || = \textstyle\cO(\sqrt{\navg})$ that holds even for sparse graphs, or equivalently few observations.
However it requires that every $M_{i,j}=\cO({||M||_1}/{k^2})$, hence 
holds only for a very limited and often impractical subclasses of parameter matrices $M$. 
They posed the question 
whether this bound also holds for general $M$. In Appendix~\ref{app: counter}, 
we provide a counterexample that answers the question in the negative. 
We construct an explicit Bernoulli parameter matrix $M$, 
s.t. \whp 
$||R(N,\swgenX)||=\Omega({\navg})$, 
much larger than $\cO(\sqrt{\navg})$.\looseness-1

Yet upper bounding 
$|| R(N,\swgenX)||$ is just one approach to achieving optimal sample complexity. 
One of this paper's contribution is an alternative approach that decomposes the noise matrix $N$ into two parts. A large part with small spectral norm, and a small part, in fact a submatrix, that may have a large spectral norm. 
While we cannot identify the "noisy" part, as shown in Section~\ref{sec:recoveralg}, we can ensure that no small part has a large influence on the estimate. 

The next theorem establishes the above partition  for all parameter matrices, and for all settings.
To specify the matrix decomposition, both here and later,
for $\gnrmtr \in \Rkm$ and
subset $S\subseteq [k]\times [m]$, let $\gnrmtr_{S}$ be the projection of matrix $\gnrmtr$ over $S$ that agrees with $\gnrmtr$ for indices in $S$ and is zero elsewhere. 
Further let $A_I:=A_{I\times [m]}$ and $A_{I\comp}:=A_{[k]\setminus I\times [m]}$ be the matrices derived from $A$ by zeroing out all rows outside, and inside, set $I$, respectively.


For the weights $\swgenX = (\snfx,\snbx)$ above, let 
$\snfx(I):=\sum_{i\in I}\snfex$ denote the \emph{weight} of row subset $I$.
\begin{theorem}\label{th:sp norm}
For $X\sim M$, any $\epsilon \ge 
\frac1{\navg}\max\big(\frac{\log^4 k} {k},{\exp^{-\frac{\navg}{8}}}\big)
$, with probability $\ge 1-6k^{-3}$,
there is a row subset $\Iellbad\subseteq [k]$ of possibly contaminated rows
with weight
$\snfx(\Iellbad)\le \epsilon k$ and
\[
|| {R(N,\swgenX)} _{\Iellbad\comp}||
\leq
\cO\big(\sqrt{\navg}\cdot\log{\textstyle\frac 2\epsilon}\big).
\]
\end{theorem}
Since $\snfex$ is the maximum of $1$ and the number of observations in row $i$, the theorem shows that for some $\Iellbad\subseteq [k]$ with only a few rows, $X_{\Iellbad}$ contains only few observations, and 
zeroing out rows $\Iellbad$ from the regularized noise matrix $R(N,\swgenX)$ would result in small spectral norm. 
We derive an algorithm that uses this result and to provably recovers low rank parameter matrices, even when the number of samples are only $\cO(k)$.


To prove Theorem~\ref{th:sp norm}, we extend a technique used in~\cite{le2017concentration} for specialized Bernoulli matrices with entries all below $\cO(\norm1 \mtrmn/k^2)$, to bound the spectral noise norm of general models and Matrices. 
In Appendix~\ref{app:propn}, we use standard probabilistic methods and concentration inequalities, to establish concentration in $\ell_{\infty}\rightarrow\ell_2$ norm for all sub-matrices of ${R(\noisemtrspt,\swgenX)}$.
We then recursively apply a form of Grothendieck-Pietsch Factorization~\cite{ledoux2013probability} 
and incorporate these bounds to partition ${R(\noisemtrspt,\swgenX)} $ into successively smaller submatrices and upper bound their spectral norms, until the resulting submatrix is very small.
Finally we show that the squared spectral norm of any matrix is at most the sum of the squared spectral norms of its decomposition parts, and thus upper bound the spectral norm of ${R(\noisemtrspt,\swgenX)}$, except for the small submatrix, that is excluded.


\section{Recovery Algorithms}\label{sec:recoveralg}
Recall that $R(A,w)$ is the regularized matrix $A$, and that 
$R(A,w)^{(r)}$ is its rank-$r$ truncated SVD. 
For any $r>0$,  regularization weights $w$, and matrix $A$, let the $(r,w)$-SVD of $A$ be the de-regularized, rank-$r$-truncated SVD of regularized matrix $A$,
\[
A^{(r,w)} := D^{\frac12}(w^\front)\cdot
R(A,w)^{(r)}
\cdot D^{\frac12}(w^\back).
\]

Let $A$ be rank-$r$ and $B$ be any matrix.
The next lemma bounds the $L_1$ distance between $A$ and  $B^{(r,w)}$ in terms of the regularized spectral distance between $A$ and  $B$. Appendix~\ref{app:LA} provides a simple proof. 

\begin{lemma}\label{th:recovery}
For any rank-$r$ matrix $\gnrmtr\in\Rkkr$, matrix $B\in\Rkk$, and weights $w$,
\begin{align*}
    \norm1{A - B^{(r,w)}}
    \leq
    \textstyle\sqrt{r\cdot\paren{{\textstyle\sum}_{i}  w^\front(i)}(\sum_{j}  w^\back(j))}\cdot \norm{}{R(A-B,w)}.
\end{align*}
\end{lemma}

Recall that $N=X-M$. The lemma implies that when $||R(\noisemtrspt,\swgenX)||$ is small, $X^{(r,\swgenX)}$, the $(r,\swgenX)$-SVD of $X$, would recover $M$.
However, $||R(\noisemtrspt,\swgenX)||$ could be large. Instead, Theorem~\ref{th:sp norm} in the last section implies that \whp, the following \emph{essential property} holds.


\textbf{Essential property\quad}
There is an unknown \emph{contaminated} set, $\Iellbad\subset [k]$, such that after zeroing all $\Iellbad$ rows in $X$ and $M$, their regularized spectral distance is at most,\looseness-1
\begin{align}\label{con:main}
||R(\noisemtralg,\swgenX)_{{\Iellbad\comp}}||\le \tau:= \cO(\sqrt{\navg} \log (r\navg)),
\end{align}
and the weight of the set $\Iellbad$ is at most, $ \snfx(\Iellbad) \le \wcon := \cO( k/(r\navg)^2)$.
The above property holds with probability $\ge 1-6k^{-3}$, by choosing $\epsilon =1/(r\navg)^2$ in Theorem~\ref{th:sp norm}. 

This property implies a small regularized spectral distance between $X_{{\Iellbad\comp} }$ and $M_{{\Iellbad\comp} }$, and since $ \snfx(\Iellbad) \le \wcon$, the ``noisier'' part of observation matrix $X_{{\Iellbad}}$ is limited to just a few rows and observations. \looseness-1



Recall that the $(r,w)$-SVD of any matrix $A$ is the de-regularized rank-$r$-truncated SVD of regularized matrix $R(A,w)$.
If the set $\Iellbad$ of contaminated rows was known, a simple $(r,\swgenX)$-SVD of $X_{{\Iellbad\comp} }$ would recover $M_{{\Iellbad\comp} }$, and since the rows $\Iellbad$ of $X$ have only a few observations, recover $M$ as well. Lemma~\ref{th:recovery} provides the normalized $L_1$ error of this idealized estimator.

Building on this simple $(r,\swgenX)$-SVD, we derive our main algorithm, Curated SVD, that 
achieves the same performance guarantee up to a constant factor, 
even when the set $\Iellbad$ is unknown.
As Property~\eqref{con:main} holds
with high probability, the reminder of this section assumes that it holds.\looseness-1

In Curated SVD, for every row subset $I\subset [k]$ with weight $\snfx(I)\le \wcon$,  we limit the maximum impact the submatrix $R(X,\swgenX)_{I }$ can have on the truncated SVD of regularized observation matrix $R(X,\swgenX)$. In particular, this limits the impact of the unknown ``noisier'' submatrix $R(X,\swgenX)_{\Iellbad }$.

To describe the algorithm, we first  formally define the \emph{impact} of a row subset on SVD components.
For a general matrix $A\in \Rkk$, let $A = \sum_{j=1}^{k} \sigma_j u_j {v_j}^\intercal $ be its SVD. Then for any $i,\, j\in [k]$, let the \emph{impact} of row $i$ of $A$, on the $j^{th}$ component in SVD of $A$, be $\mathcal{H}(A,i,j) := \sigma^2_j u_j(i)^2$. Similarly, for a row subset $I$, the \emph{impact} of $A_{I }$, or simply impact of $I$, on $j^{th}$ component in SVD of $A$ be $\mathcal{H}(A,I,j) := \sum_{i\in I} \mathcal{H}(A,i,j)$, the sum of the impact of each row. From the standard properties of SVD, it is easy to see that $\mathcal{H}(A,I,j)= \sigma^2_j{\textstyle{\sum}_{i\in \Iell}u_j(i)^2 } = || A_{I } \cdot v_j ||^2$. 

Next, we present an overview of the algorithm and its analysis.
Essentially, Curated SVD finds a set $\Ielldel \subseteq [k]$ of row indices, such that the following \emph{objectives} are fulfilled:
\begin{enumerate}[label=\textbf{(\roman*)}]
\item\label{obj1} 
Let $(R(X,\swgenX)_{{\Ielldel\comp}})^{(2r)}=\sum_{j=1}^{2r} \sigma_j u_j {v_j}^\intercal$, be the rank $2r$-truncated SVD of regularized observation matrix upon zeroing out rows $\Ielldel$. For every row subset $I$ of weight $\snfx(I)\le\wcon$ and $\forall\, j\in [2r]$, the impact of $I$ on the $j^{th}$ component is $\mathcal{H}(R(X,\swgenX)_{{\Ielldel\comp} },I,j)= \sigma^2_j{\textstyle{\sum}_{i\in \Iell}u_j(i)^2 }\le 16\tau^2$.
\item{\label{obj2}} The total weight of the set $\Ielldel$ is small,  $\snfx(\Ielldel)=\sum_{i\in\Ielldel} \snfex \le \cO(k/\navg)$.
\end{enumerate}
After finding such a row subset $\Ielldel$, Curated SVD zeroes out rows $\Ielldel$ from $X$ and simply returns the  $(2r,\swgenX)$-SVD of $X_{{\Ielldel\comp} }$ as the estimate of $M$.\looseness-1

Objective~\ref{obj1}, ensures that for every collection of rows $I$ of weight $\snfx(I)\le \wcon$, the impact of 
$R(X,\swgenX)_{I}$ on each of the first $2r$ components of SVD is small. 
As weight of $\Iellbad$ is at most $\wcon$, in particular it limits the impact of the noisier submatrix $R(X,\swgenX)_{\Iellbad }$.

Objective~\ref{obj2} ensures that the weight of $\Ielldel$, and hence the number of observations in $X_{\Ielldel}$, that get ignored in final truncated regularized SVD is small. This limits the loss due to the ignored rows $\Ielldel$.

{
Lemma~\ref{lem:algper} in the appendix uses these two observations to show that, the $(2r,\swgenX)$-SVD of $X_{{\Ielldel\comp}}$ recovers {$M$}.}

Next, we describe the algorithm Curated-SVD and show that it finds a set $\Ielldel$ that achieves Objective~\ref{obj1} and~\ref{obj2}.
The pseudo-code of the algorithm is in Appendix~\ref{ap:code}.
\subsection{Description of the Curated SVD}
Curated-SVD starts with $ \Ielldel  = \phi$.
In each iteration it calculates 
$\big(R(X,\swgenX)_{\Ielldel\comp }\big)^{(2r)} = \sum_{j=1}^{2r} \sigma_j u_j {v_j}^\intercal  $, the rank-$2r$ truncated SVD of $R(X,\swgenX)_{\Ielldel\comp }$. 
Then it callssubroutine~\emph{Row-Deletion} for each $j\in [2r]$.\looseness-1

Row deletion checks for the row subsets $I$ with small weight and high impact that violate Objective~\ref{obj1}. It then adds rows from such subsets $I$ to $\Ielldel$ to reach Objective~\ref{obj1} in a way that $\Ielldel$ does not end up too heavy.
To do this, Row-Deletion tries to find a row subset $\Iell\subseteq \Ielldel\comp$ with weight $\snfx(I)\le\wcon$ and maximum impact $\mathcal{H}(R(X,\swgenX)_{{\Ielldel\comp} },I,j)= \sigma^2_j{\textstyle{\sum}_{i\in \Iell}u_j(i)^2 }$. This however is essentially the well-known NP-hard $0$-$1$-knapsack problem.
Instead, we use a greedy 0.5-approximation algorithm~\cite{sarkar1992simple} for 0-1 knapsack, to obtain a row subset $I$ with weight $\le \wcon$, such that its impact is at least half of the maximum possible, namely
\[\textstyle \mathcal{H}(R(X,\swgenX)_{{\Ielldel\comp} },i,j)\ge 0.5 \max\{\mathcal{H}(R(X,\swgenX)_{{\Ielldel\comp} },I',j): I'\subseteq \Ielldel\comp, \sum_{i\in I'}\snfex \le \wcon\}.\]
If the impact of row collection $I$, found using the greedy algorithm, is $\mathcal{H}(R(X,\swgenX)_{{\Ielldel\comp} },I,j)< 8\tau^2$, sub-procedure Row-Deletion terminates. 
Else it adds a row $i\in I$ to $\Ielldel$, with probability of row $i\in I$ proportional to its impact to the weight ratio, $\mathcal{H}(R(X,\swgenX)_{{\Ielldel\comp} },i,j)/\snfex$.
Row-Deletion repeats this procedure on the remaining rows in $\Ielldel\comp$, until it terminates.

After calling Row-Deletion for each $j\in [2r]$, Curated-SVD checks if the new rows were added to $\Ielldel$ in this iteration, in which case it repeats the same procedure in the next iteration with updated $\Ielldel$, else it returns $(2r,\swgenX)$-SVD of $X_{\Ielldel\comp}$ as the estimate of $M$.

\textbf{Curated SVD achieves Objective~\ref{obj1}.\quad} Iterations in Curated-SVD stop when for the current choice of $\Ielldel$, for all $j\in[2r]$, Row-Deletion fails to add any row to $\Ielldel$.
This happens when for each $j\in[2r]$, the greedy-approximation algorithm in Row-Deletion finds the row subset $I$ that has impact 
$< 8\tau^2$, which implies that the impact of every row subsets $\Iell'\subseteq \Ielldel\comp$ of weight $\snfx(I')\le\wcon$, is at most $16\tau^2$. Therefore, iterations in the algorithm stops only when $\Ielldel$ meets Objective~\ref{obj1}.


\textbf{Curated SVD achieves Objective~\ref{obj2} w.p. $1-\cO(k^{-2})$.\quad} This is more challenging of the two objectives. The proof is based on the following key observation proved in the Appendix~\ref{proofobj2}. 

There exist a row subset $\Iellstr$ s.t. for every subset $I\subseteq \Iellstr\comp$ of weight $\snfx(I)\le \wcon $, the following holds $|| R(X_{I},\swgenX) || \le 2\tau.$ And moreover, the weight of $\Iellstr$ is $\snfx(\Iellstr)\le\cO(k/\navg)$. 

It is easy to see that $|| R(X_{I},\swgenX) ||^2$ upper bounds the impact submatrix $R(X_{I},\swgenX)$ can have on the components of SVD. Therefore, from the previous observation, any row subset $I$ of weight $\snfx(I)\le \wcon $ with impact $\ge 8\tau^2$, must have more than half of its impact due to the rows $I\cap\Iellstr$. 
Recall that Row-Deletion adds a row from $I$ to $\Ielldel$ only when $I$ has impact $\ge 8\tau^2$. We show that in each step in expectation it adds more weight from $I\cap\Iellstr$ to $\Ielldel$, then from $I\cap\Iellstr^c$.  

Using this we show that, \whp the total weight of rows added to $\Ielldel$ is only a constant times the weight of $\Iellstr$, that is $\cO(k/\navg)$, and Objective~\ref{obj2} is achieved
with constant probability. Repeating the procedure 
$\cO(\log k)$ times on the same data, Objective~\ref{obj2} holds with probability $1-\cO(k^{-2})$.


{ Finally, by combining Lemma~\ref{lem:algper} and the fact that the Curated SVD achieves both objectives, we prove Theorem~\ref{th:main2} in Appendix~\ref{subsec:main2}.}

\section*{Broader impact}
This work does not present any foreseeable societal consequence.


\section*{Acknowledgements}

We thank Vaishakh Ravindrakumar and Yi Hao for their helpful comments in the prepration of this manuscript.

We are grateful to the National Science
Foundation (NSF) for supporting this work through grants
CIF-1564355 and CIF-1619448.

\bibliographystyle{alpha}
\bibliography{ref}

\newpage

\onecolumn
\appendix
\section{Preliminaries} 
This section contains some standard linear-algebra results used in the proofs. It shares terminology with Section~\ref{sec:prelim}.
The \emph{Frobenius}, or \emph{entry-wise $L_2$, norm} of a matrix $\gnrmtr$ is 
\[
\norm F\gnrmtr := \sqrt{\textstyle{\sum}_{i,j}|\gnrmtr_{ij}|^2}.
\]

\begin{theorem} \emph{(Singular value decomposition~\cite{bhatia2013matrix})} For all $\gnrmtr \in \Rkm$,
\begin{enumerate}
    \item $\gnrmtr v_i(\gnrmtr) = \sigma_i(\gnrmtr )u_i(\gnrmtr) $ and $\gnrmtr ^\intercal  u_i(\gnrmtr) = \sigma_i(\gnrmtr )v_i(\gnrmtr),\ \forall\, i$. 
    \item $||\gnrmtr || = \sigma_1(\gnrmtr )$.
    \item \textit{(Courant-Fischer Theorem for
Singular Values)}: 
\[ 
\sigma_i(\gnrmtr ) = \min_{S:dim(S)=m-i+1}\ \max_{v\in S:||v||=1} ||\gnrmtr v|| = \max_{S:dim(S)=i} \  \min_{v\in S:||v||=1} ||\gnrmtr v||.
\]
\end{enumerate}
\end{theorem}

The singular values of a submatrix are smaller than those of the original matrix. In particular, the same holds for their spectral norms. 

\begin{theorem} \emph{(Interlacing property of Singular Values \cite{queiro1987interlacing})}\label{th:interlace}
For $\gnrmtr \in \Rkm$, $I\subseteq [k]$, $J\subseteq [m]$ let $\gnrmtr' = \gnrmtr _{I\times J}$ be a submatrix of $A$, then $\forall\, i$,
\[
\sigma_i(\gnrmtr') \leq \sigma_i(\gnrmtr ).
\]
\end{theorem}

Singular values are subadditive.
\begin{theorem}\label{th:weyl} \textit{(Weyl's Inequality for Singular Values~\cite{bhatia2013matrix})} 
For all $\gnrmtr, \gnrmtr' \in \Rkm $ and  $i,j\ge 1$ s.t. $i+j-1\le\min(m,n)$,
 \[
 \sigma_{i+j-1}(\gnrmtr+\gnrmtr')
 \le
 \sigma_i(\gnrmtr )+\sigma_j(\gnrmtr').
 \]
\end{theorem}

\section{Pseudo Code for the Algorithm}\label{ap:code}

\begin{algorithm}[H]
  \caption{CURATED-SVD}\label{SVD}
  \textbf{Input} : Matrix $X$, $\wcon$ and $\tau$ \\
  \textbf{Output} : Matrix $\widehat{M}$
  \begin{algorithmic}[1]
    \STATE{$I^{\del} \gets \phi$}
        \WHILE{\texttt{true} } 
            \STATE {${\mathcal U}$ $\gets  R(X,\swgenX)_{{\Ielldel\comp} }$ \COMMENT{Recall $ R(X,\swgenX) = \snf\cdot X\cdot \snb$}} 
            \STATE \text{Perform rank $2r$ truncated SVD on {$\mathcal U$}  to get $\mathcal U^{(2r)}=\sum_{j=1}^{2r} \sigma_j u_j {v_j}^\intercal $}
            \STATE $\Ielldelold \gets \Ielldel $ 
            \FOR{$j\in [2r]$ }
                \STATE $\textsc{Row Deletion}(\sigma_j \cdot u_j,\, \snfx,\, \Ielldel, \,8\tau^2,\, \wcon)$ \COMMENT{to update ${\Ielldel}$} 
            \ENDFOR 
            \IF{$\Ielldelold == \Ielldel $}  
                \STATE Break;
            \ENDIF
        \ENDWHILE
    \STATE $M^{\text{cur}}\gets D^{\frac{1}{2}}(\snfx)\cdot \mathcal U^{(2r)} \cdot D^{\frac{1}{2}}(\snfx)$
    \COMMENT{Same as $(2r,\swgenX)$-SVD of {$X_{{\Ielldel\comp} }$}}
    \STATE Return( $M^{\text{cur}}$ )    
  \end{algorithmic}
\end{algorithm}

\begin{algorithm}[H]
  \caption{ROW DELETION}\label{deletion}
  \textbf{Input} : $u = \sigma_j \cdot u_j,\, w = \snfx,\, \Ielldel,\, V=8\tau^2,\, W=\wcon$\\
  \textbf{Output}: {Updated $\Ielldel$}
  \begin{algorithmic}[1]
    \WHILE{ True }
        \STATE Use 0.5-approx Algorithm~\cite{sarkar1992simple} for 0-1 knapsack problem to find subset $I\subseteq \Ielldel\comp$ with total weight $\sum_{i\in I}w(i)\le W $ and impact $\mathcal{H}(I,j) \ge 0.5\max_{I'\subseteq \Ielldel\comp : \sum_{i\in I'}w(i)\le W} \mathcal{H}(I',j)$\\
        \COMMENT{\textit{Comment: Recall that $\mathcal{H}(I,j) = \sum_{i\in I} \sigma^2_j u^2_j(i) = \sum_{i\in I} u(i)^2$} } 
        \IF{ $\mathcal{H}( I,j) \le V$ }
            \STATE Return; \COMMENT{\textit{Comment: This happens only when $\max_{I'\subseteq \Ielldel\comp : \sum_{i\in I'}w(i)\le \wcon} \mathcal{H}(I',j) \le 2V$}}
        \ENDIF
        \STATE $\Ielldel\gets\Ielldel\,\bigcup$ An element $i\in I$ where the probability of picking $i$ is proportional to  $\frac{u(i)^2}{w(i)}$
    \ENDWHILE
  \end{algorithmic}
\end{algorithm}


\section{Analysis of Curated SVD algorithm
}
\label{ap:algo}
{Subsection~\ref{proofobj2} shows that \whp Curated SVD achieves Objective~\ref{obj2}. 
Subsection~\ref{d2} shows that if $\Ielldel$ achieves both} objectives then regularized Curated SVD recovers $M_{{\Iellbad\comp} }$. 


\subsection{Curated SVD achieves Objective~\ref{obj2}}\label{proofobj2}
We first show that both the objectives can be achieved simultaneously.
\begin{lemma}\label{lem:ma}
Assume that the essential property~\ref{con:main} holds. There is a row-collection $\Iellstr$ of weight $\sum_{i\in\Iellstr}\swgenX \le \cO(k/\navg)$, such that for any subset $I\subseteq \Iellstr\comp$ of weight $\sum_{i\in I}\snfex\le \wcon $,
\[
|| R(X_{I},\swgenX) ||  \le 2\tau.
\]
\end{lemma}

\paragraph{Implication of the above Lemma} It is easy to see that for any subset $I\subseteq [k]$, $|| R(X_{I},\swgenX) ||^2$ upper bounds the impact $R(X_{I},\swgenX)$ can have on the SVD. Hence, if we let $\Ielldel=\Iellstr$ then both the objectives will be satisfied. Although, the above lemma shows the existence of $\Iellstr$ with small weight, it doesn't give a computationally efficient way to find $\Iellstr$. We later show that Curated-SVD finds $\Ielldel$ efficiently, with expected weight only twice that of $\Iellstr$.

In proving the above lemma, we use the following two auxiliary lemmas, which are stated for general matrices, and the proofs of both the auxiliary lemmas appear in the Appendix~\ref{app:LA}, with the proof of many other Linear algebra results used in the paper.

\begin{lemma}\label{dampcon}
For any matrix $\gnrmtr$ and weight vectors $w^\front$ and $w^\back$ with positive entries
\[
||D^{-\frac12}(w^\front)\cdot\gnrmtr\cdot D^{-\frac12}(w^\back)|| \leq  \sqrt{\max_{i}\frac{||\gnrmtr_{i,*}||_1}{w^\front(i)} \times \max_{j}\frac{||\gnrmtr_{*,j}||_1}{w^\back(j)}}.
\]
\end{lemma}

We use the above lemma to bound the spectral norm of the regularized observation matrix $R(X,\swgenX)$. Applying the lemma for matrix $A=X$ and weights $w^\front=\snfx$ and $w^\back=\snbx$ we get:
\begin{align}\label{eq:hw}
||R(X,\swgenX)|| \le \sqrt{\max_{i}\frac{||X_{i,*}||_1}{\snfex} \times \max_{j}\frac{||X_{*,j}||_1}{\snbex}} \le \navg 
.
\end{align}

Next, we state the second auxiliary lemma required in proving Lemma~\ref{lem:ma}
\begin{lemma}\label{lem:number}
Let $\gnrmtr $ be an $k\times m$ matrix such that $\sigma_{1}(\gnrmtr )\leq \alpha $ and $\sigma_{r+1}(\gnrmtr )\leq \beta$. Then the number of disjoint row subsets ${I}\subset [k] $ such that $||\gnrmtr _{{I} }||> 2\beta$ is at most $\Big(\frac{r\alpha}{\beta}\Big)^2 $.
\end{lemma}

Using the above two auxiliary lemmas, we prove Lemma~\ref{lem:ma}.

\textit{Proof of Lemma~\ref{lem:ma}.}
Using Weyl's inequality~\ref{th:weyl}
\begin{align}
\sigma_{r+1}(R(X,\swgenX)_{{\Iellbad\comp} })&\le \sigma_{r+1}(R(M,\swgenX)_{{\Iellbad\comp} }) +\sigma_{1}(R(X-M,\swgenX)_{{\Iellbad\comp} })\nonumber\\
&\overset{\text{(a)}}\le \sigma_{1}(R(N,\swgenX)_{{\Iellbad\comp} })= ||R(N,\swgenX)_{{\Iellbad\comp} }||\le \tau,\label{sad2}
\end{align}
here (a) uses $X-M=N$, and the fact that $M_{{\Iellbad\comp} }$ has rank $r$. 

From equation~\eqref{eq:hw} and Theorem~\ref{th:interlace}, 
\begin{equation}\label{sad22}
     \sigma_{1}(R(X,\swgenX)_{{\Iellbad\comp} }) \leq \navg 
     .
\end{equation}
Applying Lemma \ref{lem:number} for $\gnrmtr = R(X,\swgenX)_{{\Iellbad\comp} }$, and using \eqref{sad2} and \eqref{sad22} shows that the number of disjoint subsets $I\subset\Iellbad\comp$ such that $|| R(X_{I},\swgenX) || > 2\tau$ is at most $\Paren{\frac{r\navg}{\tau}}^2$.

Therefore, the number of disjoint subsets $I\subset \Iellbad\comp$, such that $|| R(X_{I},\swgenX) || > 2\tau$ and weight $\snfx(I)\le \wcon$, is also at most $\Paren{\frac{r\navg}{\tau}}^2$. Since every such subset $I$ has weight $\snfx(I)\le \wcon$, the combined total weight of all such disjoint subsets is $\le \wcon\cdot \Paren{\frac{r\navg}{\tau}}^2$.

Let $\Iellstr\subseteq[k] $ be the subset formed by combining all these disjoint subsets, and $\Iellbad$. It is clear from the construction of $\Iellstr$ that, for every subset $I\subset \Iellstr\comp$ with weight $\snfx(I)\le \wcon $,
\[
 || R(X_{I},\swgenX) || \le 2\tau.  
\]
Finally, noting that the total weight of subset $\Iellstr$ is at most $ \wcon\cdot\Paren{\frac{r\navg}{\tau}}^2 + \sum_{i\in \Iellbad}\snfex \le \wcon \cdot\Paren{\frac{r\navg}{\tau}}^2+\wcon$, and using $\wcon= \cO(k/(r\navg)^2)$, $\tau = \cO(\sqrt{\navg}\log (r\navg))$ completes the proof.
\hfill$\blacksquare$



The next lemma bounds the {expected weight} of set $\Ielldel$ zeroed out by Curated SVD, by showing that it is at most twice the weight of $\Iellstr$. 
\begin{lemma} Assume that essential property~\ref{con:main} holds. When Curated SVD terminates, the total weight of the final set of rows $\Ielldel$ satisfy,
\[
\E\Big[ \snfx(\Ielldel) \Big] \le \cO(k/\navg).
\]
\end{lemma}
\begin{proof}
Recall that the rows get added to $\Ielldel$ one by one. Let $Z_t$ denotes the weight of the $t^{th}$ row that gets added to $\Ielldel$. Let indicator random variable $\mathbbm{1}_t(\Iellstr) = 1$, if the $t^{th}$ that gets added in $\Ielldel$ belongs to $\Iellstr$ and similarly indicator random variable $\mathbbm{1}_t(\Iellstr\comp) = 1$, if the $t^{th}$ that gets added in $\Ielldel$ doesn't belongs to $\Iellstr$. Clearly $Z_t = Z_t\cdot\mathbbm{1}_t 
(\Iellstr)+Z_t \mathbbm{1}_t(\Iellstr\comp) $. 
We first show that for any $t$,
\[
\E[Z_t\cdot\mathbbm{1}_t(\Iellstr\comp)]\le \E[Z_t\cdot\mathbbm{1}_t(\Iellstr)]. 
\]
Suppose that the $t^{th}$ row added to $\Ielldel $ was chosen probabilistically from the row subset $I\subseteq \Ielldel\comp$ in Row deletion in step (6). Recall that Row-Deletion adds rows only from the row subsets $I$ that have weight $\le\wcon$ and impact at least $8\tau^2$, on one of the components of SVD.
Then $\snfx(I)\le \wcon $ and some $j\in [2r] $, and its impact
\[
\mathcal{H}(I,j) > 8\tau^2.
\]
We decompose $\mathcal{H}(I,j)$ in two parts as
\begin{align*}
    \mathcal{H}(I,j) = \mathcal{H}(I \cap \Iellstr,j) +\mathcal{H}(I \setminus \Iellstr ).
\end{align*}
Since, the weight of $I$ is $\le \wcon$, the weight of $I' = I \setminus \Iellstr$ is also $\le \wcon$. Applying Lemma~\ref{lem:ma} on $I'$ implies that
\begin{align*}
    \mathcal{H}(I \setminus \Iellstr, j ) \le 4\tau^2.
\end{align*}
Combining the last three equation gives, 
\begin{align*}
    \mathcal{H}(I \cap \Iellstr,j)\ge 4\tau^2.
\end{align*}
Next, note that
\[
\E[Z_t\cdot\mathbbm{1}_t(\Iellstr)] = \sum_{i\in I \cap \Iellstr} \Pr[i^{th} \text{ row is picked }]\cdot \snfex \ {\propto} \sum_{i\in I \cap \Iellstr} \frac{\mathcal{H}(i,j)}{\snfex}\cdot \snfex = \mathcal{H}(I \cap \Iellstr,j),
\]
where we used the fact that the probability of adding row $i$ to $\Ielldel$ in Row-Deletion procedure is proportional to $\frac{\sigma_j^2u_j^2(i)}{\snfex}=\frac{\mathcal{H}(i,j)}{\snfex}$

A similar calculation implies that
\[
\E[Z_t\cdot\mathbbm{1}_t(\Iellstr\comp)] \propto  \mathcal{H}(I \setminus \Iellstr,j).
\]
Combining the last four equations show,
\[
\E[Z_t\cdot\mathbbm{1}_t(\Iellstr\comp)]\le \E[Z_t\cdot\mathbbm{1}_t(\Iellstr)].
\]

Let random variable $\ell$ denote the total number of rows that are added to $\Ielldel$ before algorithm stops. 
Then using the optional stopping theorem for supermartingale $\big(Z_t\cdot\mathbbm{1}_t(\Iellstr\comp)- Z_t\cdot\mathbbm{1}_t(\Iellstr)\big)$ implies that when algorithm stops after putting $\ell$ rows in $\Ielldel$,
\[
\E[\sum_{t=1}^{\ell} \big(Z_t\cdot\mathbbm{1}_t(\Iellstr\comp)- Z_t\cdot\mathbbm{1}_t(\Iellstr)\big)]\le 0.
\]
And since, the total weight of the rows in $\Iellstr$ that gets added to $\Ielldel$ is at-most the total weight of rows in $\Iellstr$, using Lemma~\ref{lem:ma}, we get 
\[
\sum_{t=1}^{\ell} Z_t\cdot\mathbbm{1}_t(\Iellstr)\le \cO(k/\navg).
\]
Combining the above two equations bounds the expected total weight that gets added to $\Ielldel$
\[
\E[\sum_{t=1}^{\ell} Z_t] = \E[\sum_{t=1}^{\ell} Z_t\cdot\mathbbm{1}_t(\Iellstr\comp)]+ \E[\sum_{t=1}^{\ell} Z_t\cdot\mathbbm{1}_t(\Iellstr)]\le 2\E[\sum_{t=1}^{\ell} Z_t\cdot\mathbbm{1}_t(\Iellstr)]\le 2\cO(k/\navg).\qedhere
\]

\begin{lemma}\label{reprun} Assume that essential property~\ref{con:main} holds. If Curated SVD is run $\cO(\log k)$ times, then w.p. $>1-k^{-O(1)}$, at-least one of the runs finds $\Ielldel$ s.t.,
\[
\snfx(\Ielldel) \le \cO(k/\navg).
\]
\end{lemma}
\begin{proof}
From Markov's inequality and the previous lemma, we get 
\[
\Pr[\snfx(\Ielldel) \ge  5\times \cO(k/\navg)\Big)]\le \frac{\cO(k/\navg)}{5\times\cO(k/\navg)}\le 1/5.
\]
Therefore, w.p. $\ge 4/5$, Curated SVD finds $\Ielldel$ such that,
\[
\snfx(\Ielldel) \le 5\times\cO(k/\navg).
\] 
Hence, if we run Curated SVD for $\cO(\log k)$ times, with probability $1-k^{-O(1)}$, at-least one of the runs will find $\Ielldel$ that satisfy the above equation and Objective~\ref{obj2}.
\end{proof}


\end{proof}

\subsection{Objectives~\ref{obj1} and~\ref{obj2} implies recovery}\label{d2}

\begin{lemma}\label{lem:algper}
Assume that essential property~\eqref{con:main} holds and let $\Ielldel$ be any row subset satisfying objectives~\ref{obj1} and~\ref{obj2} and let $\widehat M$ be $(2r,\swgenX)$-SVD of $X_{\Ielldel\comp}$, then 
\[
||M_{\Iellbad^c}-\widehat M||_1\le \cO(k\sqrt{r\navg} \log(r\navg)).
\]
\end{lemma}
The rest of this subsection proves the lemma.

Let $Z := R(X_{\Ielldel\comp},\swgenX)^{(2r)}$, be rank $2r$ truncated SVD of regularized matrix $R(X_{\Ielldel\comp},\swgenX)$. 
Then, from the definition of $(2r,\swgenX)$-SVD,  
\[
\mtrnew = D^{\frac{1}{2}}(\snfx)\cdot  Z\cdot  D^{\frac{1}{2}}(\snbx).
\]
Clearly, both $Z$ and $\mtrnew$ have rank $\le 2r$ and their rows $\Ielldel$ are all zero. 

To bound the total loss in Lemma~\ref{lem:algper}, using the triangle inequality, we first decompose the loss incurred in estimating matrix $\mtrmn$ by $\mtrnew$ into three parts as follows.
\begin{align}\label{3part}
&\norm1{\mtrmn_{\Iellbad^c}-\mtrnew}= \norm1{\mtrmn_{\Iellbad^c}-\mtrnew_{{\Ielldel\comp} }}\nonumber\\
&\leq 
\norm1{(\mtrmn-\mtrnew)_{\Iellbd\comp }}+
\norm1{\mtrmn_{\Ielldel\setminus \Iellbad }}+\norm1{\mtrnew_{\Iellbad\setminus\Ielldel }}.
\end{align}
Lemma \ref{lem112} bounds the contribution of the first term above on the right. 
Intuitively, the contribution of the other two terms is small, since the weight of row subsets $\Ielldel$ and $\Iellbad$ are small. And using this observation we later bound the contribution of the last two terms.

\begin{lemma}\label{lem112}
\[\norm1{(\mtrmn-\mtrnew)_{\Iellbd \comp }}\le  \cO(k\sqrt{r\navg} \log(r\navg)). \]
\end{lemma}

To prove Lemma \ref{lem112}, we use the following lemma which bounds the spectral distance between $ Z_{\Iellbd\comp }$ and $R(M,\swgenX)_{\Iellbd\comp}$ is small.
\begin{lemma}\label{eq:2212} Let $\Ielldel\subseteq[k]$ be a row subset that satisfy Objective-1, then
\begin{equation*}
    || Z_{\Iellbd\comp }-R(M,\swgenX) _{\Iellbd\comp } ||\le 9 \tau.
\end{equation*}
\end{lemma}
In proving the above Lemma, we need the following auxiliary lemmas, which is stated for general matrices, and its proof appears in the Appendix~\ref{app:LA}
\begin{lemma}\label{lem:bad part}
Let $\gnrmtr = B+C$ and $\gnrmtr  = \sum_{i} \sigma_i(\gnrmtr )u_i {v_i}^\intercal $ be the SVD decomposition of $\gnrmtr $. And $\sigma_{r+1}(B) \leq  \beta $ and $ ||C v_i||\leq 2\beta $ for $i\in [2r]$. Then
$\sigma_{2r}(\gnrmtr ) \leq 4\beta $.
\end{lemma}

\textit{Proof of Lemma~\ref{eq:2212}.}
Using equation~\eqref{sad2} and Theorem \ref{th:interlace}, we get,
\begin{align}
\sigma_{r+1}(R(X,\swgenX)_{\Iellbd\comp })\le \sigma_{r+1}(R(X,\swgenX)_{{\Iellbad\comp} })\le  \tau.
\end{align}
Observe that,
\[
R(X,\swgenX) _{\Ielldel\comp } = R(X,\swgenX) _{\Iellbd\comp }+
R(X,\swgenX) _{\Iellbad\setminus \Ielldel }.
\]

Let $(R(X,\swgenX)_{{\Ielldel\comp}})^{(2r)}=\sum_{j=1}^{2r} \sigma_j u_j {v_j}^\intercal$.
Since $\Ielldel$ satisfy Objective~\ref{obj1} and the weight of $\Iellbad$ is at-most $\snfx(\Iellbad) \le \epsilon k$, therefore, the condition in Objective~\ref{obj1} implies that
\begin{align*}
||R(X,\swgenX) _{\Iellbad\setminus \Ielldel } \cdot v_i||\le 4\tau.
\end{align*}
Then applying Lemma \ref{lem:bad part}, for $A= R(X,\swgenX) _{\Ielldel\comp }$ and using the above three equations gives
\begin{equation*}
    \sigma_{2r}(R(X,\swgenX)_{\Ielldel\comp })\le 8\tau.
\end{equation*}
Since, for $Z$ is rank-$2r$ truncated SVD of $R(X,\swgenX)_{\Ielldel\comp }$, we have
\begin{equation}\label{qws}
\norm{} {R(X,\swgenX)_{\Ielldel\comp}- Z} =\norm{} {(R(X,\swgenX)-Z)_{\Ielldel\comp}} \le 8\tau.
\end{equation}
Then,
\begin{align*}
&|| (Z-R(M,\swgenX)) _{\Iellbd\comp }|| \\
&\overset{\text{(a)}}\le
|| (Z-R(X,\swgenX)) _{\Iellbd\comp }||+|| (R(X,\swgenX)-R(M,\swgenX)) _{\Iellbd\comp }||\\
&\overset{\text{(b)}}\le
|| (Z-R(X,\swgenX)) _{\Ielldel\comp }||+|| (R(N,\swgenX)) _{\Iellbad\comp }||\\
&\overset{\text{(c)}}\le
8\tau+\tau=9\tau,
\end{align*}
here (a) uses triangle inequality, (b) uses Theorem \ref{th:interlace} and $X-M=N$, and finally (c) uses equation~\eqref{qws} and essential condition~\ref{con:main}.\hfill$\blacksquare$

The next auxiliary lemma relates the $L_1$ norm and spectral norm. Its proof appears in the Appendix~\ref{app:LA}.
\begin{lemma}\label{lem:spll}
For any rank-$r$ matrix $\gnrmtr\in\Rkm$ and weight vectors $w^\front$ and $w^\back$ with non-negative entries
\begin{align*}
    ||D^{\frac12}(w^\front)\cdot\gnrmtr\cdot D^{\frac12}(w^\back)|| _1
    \leq
    \sqrt{r  (\sum_i w^\front(i)) (\sum_j w^\back(j)) }\cdot \norm{}{\gnrmtr}.
\end{align*}
\end{lemma}
We will also need the following result.
\begin{lemma}\label{lem18}
The total weight of all rows is at most $\sum_{i\in [k]}\snfex \le 2k$ and similarly the total weight of all columns is at most $\sum_{i\in [k]}\snbex \le 2k.$  
\end{lemma}
\begin{proof}
\[
\snfx([k])= \sum_{i\in [k]}\snfex \le \sum_{i\in [k]} 1+ \frac{||X_{i,*}||_1}{\navg}\le k+ \frac{||X||_1}{\navg}\le 2k,
\]
since $\navg$ is the average number of samples in each row. Similarly it can be shown for columns.
\end{proof}

Next, combining the above result we prove Lemma~\ref{lem112}

\textit{Proof of Lemma~\ref{lem112}.}
First note that, 
\[
\norm1{(\mtrmn-\mtrnew)} = 
\norm1{M-D^{\frac{1}{2}}(\snfx)\cdot  Z\cdot  D^{\frac{1}{2}}(\snbx)}=
\norm1{D^{\frac{1}{2}}(\snfx) \Paren{Z-R(M,\swgenX)}D^{\frac{1}{2}}(\snbx)},
\]
here we used $R(M,\swgenX)= \snf\cdot M\cdot \snb$. As noted earlier $Z$ has the rank $\le 2r$ and $M$ has the rank $\le r$, therefore the rank of $ \Paren{Z-R(M,\swgenX)}$ is at most $3r$.
Then, using Lemma \ref{lem:spll} and Lemma \ref{eq:2212}, 
\begin{align*}
&\norm1{D^{\frac{1}{2}}(\snfx) \Paren{Z-R(M,\swgenX)}_{\Iellbd\comp }D^{\frac{1}{2}}(\snbx)}\\
&\le \sqrt{3r} 
\sqrt{(\textstyle\sum_{j\in [k] } \snbex)\paren{{\textstyle\sum}_{i\in  \Iellbd\comp }\snbex} }\cdot{9\tau}\\ 
&\le \sqrt{3r} 
\sqrt{(\textstyle\sum_{j\in [k] } \snbex)\paren{{\textstyle\sum}_{i\in  [k]} \snbex}}\cdot{9\tau}
\le \sqrt{3r} 
\cdot 2k\cdot{9\tau},
\end{align*}
here the last step uses Lemma~\ref{lem18}.
Combining the last two equations and using $\tau= \cO(\sqrt{\navg} \log (r\navg))$ complete the proof.\hfill$\blacksquare$

To bound the second term in \eqref{3part}, note that
\begin{align}
{\mtrmn_{\Ielldel\setminus \Iellbad }} = {D^{\frac{1}{2}}(\snfx)\cdot R(\mtrmn,\swgenX)_{\Ielldel\setminus \Iellbad } \cdot D^{\frac{1}{2}}(\snbx)} 
\end{align}
Applying Lemma \ref{lem:spll},
\begin{align}
\norm1{\mtrmn_{\Ielldel\setminus \Iellbad }}=&\norm1{D^{\frac{1}{2}}(\snfx)\cdot R(\mtrmn,\swgenX)_{\Ielldel\setminus \Iellbad }\cdot D^{\frac{1}{2}}(\snbx)}\nonumber\\
&\le 
\sqrt{r} 
\sqrt{(\textstyle\sum_{j\in [k] } \snbex)\paren{{\textstyle\sum}_{i\in  {\Ielldel\setminus \Iellbad }}  \snfex}}\cdot
||R(\mtrmn,\swgenX)_{\Ielldel\setminus \Iellbad }||\nonumber\\
&\le 
\sqrt{r} 
\sqrt{(\textstyle\sum_{j\in [k] } \snbex)\paren{{\textstyle\sum}_{i\in  {\Ielldel }}  \snfex}}\cdot
||R(\mtrmn,\swgenX)_{{\Iellbad\comp} }||\nonumber\\
&\le 
\sqrt{2rk} 
\sqrt{{  \snfx(\Ielldel)}}\cdot
||R(\mtrmn,\swgenX)_{{\Iellbad\comp} }||.
\end{align}
Next,
\begin{align*}
||R(\mtrmn,\swgenX)_{{\Iellbad\comp} }||&\le ||R(\noisemtrspt,\swgenX)_{{\Iellbad\comp} }||+||R(X,\swgenX)_{{\Iellbad\comp} }||     
\\
&\le \tau+||R(X,\swgenX)||\le \tau +\navg,   
\end{align*}
here we used essential property~\eqref{con:main} and equation~\eqref{sad22}.
Combining the last two equations we get
\begin{align}\label{swadesh}
 \norm1{\mtrmn_{\Ielldel\setminus \Iellbad }}\le \cO( \sqrt{rk\cdot {  \snfx(\Ielldel)}}\cdot (\tau +\navg))\le \cO( \sqrt{r}k\cdot (\log(r\navg) + \sqrt{\navg}))  
\end{align}
Finally, we bound the last term in \eqref{3part}. Recall that $\mtrnew = D^{\frac{1}{2}}(\snfx)\cdot  Z\cdot  D^{\frac{1}{2}}(\snbx)$. Then
\begin{align}
\norm1{\mtrnew_{\Iellbad\setminus \Ielldel }}
& = \norm1{(D^{\frac{1}{2}}(\snfx) \cdot Z_{\Iellbad\setminus \Ielldel }\cdot  D^{\frac{1}{2}}(\snbx) }. \nonumber
\end{align}
Using the fact that $Z$ is truncated SVD of $R(X_{\Ielldel\comp},\swgenX)$ and equation \eqref{sad22} we get 
\begin{equation*}
 \norm{} {Z_{\Ielldel\comp }}\le \norm{} {Z}\le \norm{} {R(X_{\Ielldel\comp},\swgenX)}  \le \navg.
 \end{equation*}
Therefore, 
\begin{equation*}
 \norm{} {Z_{\Iellbad\setminus \Ielldel }} \le \navg.
 \end{equation*}
Then applying Lemma \ref{lem:spll},  
\begin{align}
\norm1{D^{\frac{1}{2}}(\snfx) \cdot Z_{\Iellbad\setminus \Ielldel }\cdot  D^{\frac{1}{2}}(\snbx) }
&\le 
\sqrt{2r} 
\sqrt{(\textstyle\sum_{j\in [k] } \snbex)\paren{{\textstyle\sum}_{i\in  {\Iellbad\setminus \Ielldel }}  \snfex}}\cdot
\navg\nonumber\\
&\le 
\sqrt{2r} 
\sqrt{{2k}\paren{{\textstyle\sum}_{i\in  {\Iellbad}}  \snfex}}\cdot
\navg\nonumber\\
&\le 
2 \sqrt{ r k\wcon} \cdot\navg = \cO(k/\sqrt{r}),\nonumber
\end{align}
here in the last step we used $\wcon= \cO( k/(r\navg)^2)$.

By combining equation~\eqref{3part}, Lemma~\ref{lem112}, equation~\eqref{swadesh} and the above equations we get
\begin{align*}
\norm1{\mtrmn_{\Iellbad}-\mtrnew}&\le \cO(k\sqrt{r\navg} \log(r\navg))+ \cO( k\sqrt{r}\cdot (\log(r\navg) + \sqrt{\navg})) +\cO(k/\sqrt{r})\\
&\le \cO(k\sqrt{r\navg} \log(r\navg)).
\end{align*}
This completes the proof of the lemma.

\section{Proof of Theorem~\ref{th:main2}}\label{subsec:main2}
 Theorem~\ref{th:main2} follows immediately from the following theorem.
\begin{theorem} \label{th:main2s}
Curated SVD runs in polynomial time, 
and for every $k$, $r$, $\epsilon>0$, $M\in \Rkkr$, and $X\sim M$, returns an estimate $M^{\text{cur}}(X)$ s.t. with probability $\ge 1-k^{-2}$, 
\[
L(M^{\text{cur}}(X))  =\frac{||M-M^{\text{cur}}(X)||_1}{||M||_1}
\le  \cO(\sqrt{\frac {kr}{||M||_1}} \log(\frac{r||M||_1}{k})).
\]
\end{theorem}
\begin{proof}
Section~\ref{sec:recoveralg} showed that Curated-SVD always achieves Objective~\ref{obj1}. Using the spectral concentration bound in Theorem~\ref{th:sp norm} it also showed that the essential property required for Curated SVD holds with probability $\ge 1-6/k^3$. 

Lemma~\ref{reprun} showed that if essential property hold and Curated SVD is repeated $\cO(\log k)$ times, on the same samples, then w.p. $>1-k^{-O(1)}$, at-least one of the runs find $\Ielldel$ that achieves Objective~\ref{obj2}.

Finally, when essential property holds and Curated SVD achieves both the objectives then Lemma \ref{lem:algper} showed
\begin{align}\label{eqboun}
     {||M_{\Iellbad^c}-\widehat M||_1}\le \cO(k\sqrt{r\navg} \log(r\navg)).
\end{align}

Note that 
\[
||M-\widehat M||_1\le ||M_{\Iellbad^c}-\widehat M||_1+||M_{\Iellbad}||_1.
\]
Therefore, to prove the theorem the only thing remains is to bound the last term.

To bound the last term we use the following lemma. The proof of the Lemma appears in Section~\ref{lem:seccon}.

The lemma upper bounds the sum of the absolute difference between the expected and observed samples in each row of $X$.
\begin{lemma}\label{lem:absdif}
With probability $\ge 1-3k^{-3}$,
\[
\sum_{i}{|\sum_j N_{i,j}|}=\cO(k\sqrt{\navg}).
\]
\end{lemma}

Then
\begin{align}
||M_{\Iellbad}||_1&= \sum_{i\in\Iellbad}\sum_{j\in[k]} M_{i,j}\nonumber\\
&\le \sum_{i\in\Iellbad}\sum_{j\in[k]} X_{i,j} + \Bigg|\sum_{i\in\Iellbad}\sum_{j\in[k]} (M_{i,j}- X_{i,j})\Bigg|\nonumber  \\
&\le \sum_{i\in\Iellbad}||X_{i,*}|| + \sum_{i\in\Iellbad}\Bigg|\sum_{j\in[k]} N_{i,j}\Bigg|  \nonumber  \\
&\le \sum_{i\in\Iellbad}\navg\cdot \snfex +\cO(k\sqrt{\navg})  \nonumber  \\
&\le \cO(k\sqrt{\navg})  \nonumber,
\end{align}
here the second last step uses $||X_{i,*}||\le \navg\cdot \snfex$ and the previous Lemma and the last step uses $\sum_{i\in\Iellbad} \snfex\le \wcon$. Combining this with~\eqref{eqboun} and letting $M^{\text{cur}}(X) = \widehat M$ gives
\[
||M-M^{\text{cur}}(X)||_1
\le  \cO(k\sqrt{r\navg} \log(r\navg)).
\]
Finally, dividing the both sides by $||M||_1$ and using $\navg=||M||_1/k$ in the above equation completes the proof.
\end{proof}

Using $||M||_1 = \frac{kr}{\epsilon^2}\cdot \log^2\frac{r}{\epsilon}$ in the  above theorem gives Theorem~\ref{th:main2}.

The next subsection gives the implication of the above result for collaborative filtering.
\subsection{Collaborative filtering}\label{app:collf}

For the same general bounded noise model~\cite{borgs2017thy} derived the mean square error $\sum_{i,j}(F_{i,j}-\hat F_{i,j})^2/k^2 = \cO(r^2/(pk)^{2/5})$. They assume that the mean matrix $F$ is generated by a Lipschitz latent variable model.
Here we show that Curated SVD achieves a better accuracy, in a stronger norm, and without the additional Lipschitz assumption on $F$. 

Note that since $0\le F_{i,j}\le 1$, $|F_{i,j}-\hat F_{i,j}|\ge |F_{i,j}-\hat F_{i,j}|^2$.
Therefore, $L_1$ error $\sum_{i,j}|F_{i,j}-\hat F_{i,j}|/k^2$ upper bounds mean squared error.

Recall that in collaborative filtering model $M_{i,j} = p F_{i,j}$. Since the sampling probability $p$ can be estimated to very good accuracy hence without loose of generality assume that $p$ is known. 
Note that $||M||_1 = p ||F||_1 = p k^2 F_{i,j}^{\text{avg}} $, where $F_{i,j}^{\text{avg}} = ||F||_1/k^2\le 1 $ as $\forall i,j$, $F_{i,j}\le 1$.
We let $F^{\text{cur}}(X) = M^{\text{cur}}(X)/p$. 

Then using Theorem~\ref{th:main2s} for this model implies:
\[
\frac{||M-M^{\text{cur}}(X)||_1}{||M||_1}= \frac{p||F-F^{\text{cur}}(X)||_1}{p k^2 F_{i,j}^{\text{avg}}}
\le  \cO\Bigg(\sqrt{\frac {kr}{p k^2 F_{i,j}^{\text{avg}}}} \log(\frac{r }{k} p k^2 F_{i,j}^{\text{avg}})\Bigg).
\]
Therefore,
\[
\frac{||F-F^{\text{cur}}(X)||_1}{k^2}
\le  \cO\Bigg(\sqrt{\frac {r F_{i,j}^{\text{avg}}}{p k }} \log({r }\cdot p k F_{i,j}^{\text{avg}})\Bigg)\le \cO\Bigg(\sqrt{\frac {r }{p k }} \log({r } p k )\Bigg).
\]

We get the following Corollary.
\begin{corollary}
Curated SVD runs in polynomial time, 
and for every $k$, $r$, $\epsilon>0$, sampling probability $p$, $F\in \Rkkr$, $F_{i,j}\in[0,1]$ and $X\sim p F$, returns an estimate $F^{\text{cur}}(X)$ s.t. with probability $\ge 1-k^{-2}$, 
\[
\frac{||F-F^{\text{cur}}(X)||_1}{k^2}
\le \cO\Bigg(\sqrt{\frac {r }{p k }} \log({r } p k )\Bigg).
\]
\end{corollary} 

Note that the above bound on the $L_1$ error norm is strictly better than the previous bound on the mean square error, and, as we showed, MSE is also a weaker error norm than $L_1$ for this setting.



\section{Properties of the Noise matrix}\label{app:propn}
Here we give the proof of Theorem~\ref{th:sp norm}.
We in fact prove a somewhat more general version of the theorem. Accordingly, we define the generalisation of the weights $\swgen$ and $\swgenX$ defined in the paper. For any $\specparaN> 0$, define weights $\wgen := (\nfm,\nbm)$ such that,
\[\nfem := \textstyle\max\{1, \frac{||M_{i,*}||_1}{\specparaN}\} \text{ and }\nbem := \max\{1, \frac{||M_{*,j}||_1}{\specparaN}\}. \]
And similarly define $\wgenX = (\nfx,\nbx)$ such that,
\[\nfex := \textstyle\max\{1,  \frac{||X_{i,*}||_1}{\specparaN}\} \text{ and }\nbex := \max\{1, \frac{||X_{*,j}||_1}{\specparaN}\}.\]
We obtain the results for the regularization weights $\nfem$ and $\nfex$. Note that the regularization weights $\swgen$ and $\swgenX$, used in the main paper, are a special case of these regularization weights $\nfem$ and $\nfex$ for $\specparaN = \navg$.

The next theorem is a generalisation of Theorem~\ref{th:sp norm}. This theorem bounds the spectral norm of the noise matrix regularized by weights $\wgenX$, and Theorem~\ref{th:sp norm} can be obtained as a special case of this theorem for  $\specparaN = \navg$.

\begin{theorem}\label{th:sp norm gen}
For $X\sim M$, any $\specparaN> 0$, $\epsilon \ge 
\frac1{\specparaN}\max\big(\frac{\log^4 k} {k},{\exp^{-\frac{\specparaN}{8}}}\big)
$, with probability $\ge 1-6k^{-3}$,
there is a row subset $\Iellbad\subseteq [k]$ of possibly contaminated rows
such that
$\sum_{i\in I}\nfex\le \epsilon k$ and
\[
|| {R(N,\wgenX)} _{\Iellbad\comp}||
\leq
\cO\big(\sqrt{\specparaN}\cdot\log{\textstyle\frac 2\epsilon}\big).
\]
\end{theorem}

To prove the above theorem we first establish the bound on $\ell_{\infty}\rightarrow\ell_2$ norm of submatrices of the regularized noise matrix and use a known result, referred as Grothendieck-Pietsch factorization, to relate this norm to spectral norm. Next we define $\ell_{\infty}\rightarrow\ell_2$ norm and state Grothendieck-Pietsch factorization.

The $\ell_{\infty}\rightarrow\ell_2$ norm of a matrix $\gnrmtr \in \Rkm$ is
\[
||\gnrmtr ||_{\infty\rightarrow 2} := \max_{||v||_\infty=1}||\gnrmtr v||_2 = \max_{v\in\{-1,1\}^m}||\gnrmtr v||_2.
\]
Since the vector $v$ in this definition takes value in the finite set $(\{-1,1\}^m)$, standard probabilistic techniques are better suited for bounding the $\ell_{\infty}\rightarrow\ell_2$ norm than for bounding the spectral norm directly. In turn, Grothendieck-Pietsch factorization helps us relate $\ell_{\infty}\rightarrow\ell_2$ and spectral norm.

\begin{theorem}\label{th:groth}
\emph{(Grothendieck-Pietsch factorization)}\\
For any $\gnrmtr \in \Rkm$ there is a vector $\mu=(\mu(1),...,\mu(m))$ with $\mu(j)\ge 0$ and $\sum_j\mu(j)=1$ such that
\[
||\gnrmtr \cdot  D^{-\frac12}(\mu)||
\le
\sqrt{\frac{\pi}{2}}\cdot\norm{\infty\rightarrow 2}{\gnrmtr}.
\]
\end{theorem}
The above result can be obtained by combining Little Grothendieck Theorem and Pietsch Factorization, and has appeared in~\cite{ledoux2013probability} (Proposition 15.11) and~\cite{le2017concentration} (Theorem 3.1).

To prove the above theorem, therefore, we first bound the $\ell_{\infty}\rightarrow\ell_2$ norm of submatrices of $\nf\cdot {\noisemtrspt}$.
The proof of the lemma is based on standard use of the probabilistic methods. 
Due to the symmetry, a similar bound will hold on $\ell_{\infty}\rightarrow\ell_2$ norm of submatrices of $(\noisemtrspt\cdot \nb)^\intercal = \nb\cdot \noisemtrspt^\intercal$.

\begin{lemma} \label{inf2}($\ell_\infty\rightarrow \ell_2$ concentration) With probability $\ge 1-3k^{-3}$, for
every $\max(\frac{\log^4 k}{\specparaN},\frac{k \exp^{-\frac{\specparaN}{8}}}{\specparaN})\leq \ell\leq k $, and every $I, J\subseteq [k]$ of size $\ell$, 
\[
||(\nf\cdot {\noisemtrspt})_{I\times J}||_{\infty\rightarrow 2}\leq \cO(\sqrt{\specparaN \ell \log (e k/\ell)}).
\]
\end{lemma}
\begin{proof}
\begin{align*}
||( {\nf\cdot \noisemtrspt})_{I\times J}||^2_{\infty\rightarrow 2}&= \max_{v\in \{-1,1\}^\ell}||( {\nf\cdot \noisemtrspt})_{I\times J} \cdot  v||^2\\
&= \max_{v\in \{-1,1\}^\ell}\sum_{i\in I}\Big(\sum_{j\in J}\frac{\elenoisemtrspt}{\sqrt{\nfex}} v(j)\Big)^2\\
&= \max_{v\in \{-1,1\}^\ell}\sum_{i\in I}\frac{Z_i(v)^2}{\nfex} =\max_{v\in \{-1,1\}^\ell}\sum_{i\in I}{\widehat{Z}_i(v)^2},
\end{align*}
where 
\[
Z_i(v) = \sum_{j\in J}\elenoisemtrspt v(j) \quad \text{and}\quad \widehat{Z}_i(v)= \frac{Z_i(v)}{\sqrt{\nfex}}  .\] 

For a fix $v$, $Z _i(v)$ is the sum of independent zero-mean random variables, the following bound follows from  Bernstein’s inequality, for any $t>0$
\begin{align}
\Pr( |Z _i(v)| > t ) \leq 2\exp\Bigg(  \dfrac{-t^2/2}{||M_{i,*}||_1+t/3}\Bigg)\le 2\exp\Bigg(  \dfrac{-t^2/2}{\specparaN\nfem+t/3}\Bigg).\label{eq:conber}
\end{align}

Observe that 
\begin{align}
|Z_i(v)|\le \sum_{j\in J} |N_{i,j}|\le ||N_{i,*}||_1 \le ||X_{i,*}||_1+||M_{i,*}||_1\leq   (\nfex+\nfem)\specparaN
.\label{eq:whenmed}
\end{align}



Based on the values of $\nfem$ and $\nfex$ we divide the rows into two categories, and let random variable $T_i$ denote the category of row $i$, 
\[
T_i := 
\begin{cases}
1 &  \text{ if } \nfex\ge \frac\nfem2,\\
2 &  \text{ else }.
\end{cases}
\]
Let ${\xi_i}:=\mathbbm{1}_{\{T_i=1\}} $ be the indicator random variable corresponding to the event that row $i$ is in the first category. Hence, ${\bar\xi_i} := 1-{\xi_i} =\mathbbm{1}_{\{T_i=2\}} $. Then
\[
\sum_{i\in I} \widehat{Z}_i(v)^2 = \sum_{i\in I} (\xi_i+ \bar\xi_i) \widehat{Z}_i(v)^2.
\]
Next, we bound the contribution of the rows in each categories to the above term.

\begin{enumerate}
\item 
$T_i=1:\ \nfex\ge \frac\nfem2 $. \\
To bound the contribution of the rows in category 1, we show for any given $v$, ${\xi_i} \widehat{Z}_i(v)$ are sub-Gaussian random variables with sub-Gaussian norm $\cO( \sqrt{\specparaN})$. Then we bound sum of their squares, which are sub-exponential random variable, by applying Bernstien's concentration bound. We first prove that ${\xi_i} \widehat{Z}_i(v)$ are sub-Gaussian.

From equation~\eqref{eq:whenmed} and the definition of category 1, we get
\[
|{\xi_i}Z_i(v)|\le |Z_i(v)| \le(\nfex+\nfem)\specparaN\le (\nfex+2\nfex)\specparaN= 3\nfex\specparaN,
\]
hence
\[
\nfex\ge \frac{|{\xi_i}Z_i(v)|}{\specparaN}.
\]
Then
\[
|{\xi_i}\widehat{Z}_i(v)|=\frac{|{\xi_i}Z_i(v)|}{\sqrt{\nfex}}\le \min\Big\{|\sqrt{\specparaN Z_i(v)}|,\frac{|Z_i(v)|}{\sqrt{\nfem/2}}\Big\},
\]
here we used the previous equation, the fact that ${\xi_i}\le 1$, and $\nfex\ge \nfem/2 $, which follows from the definition of the category 1. 
Using equation~\eqref{eq:conber} we get
\begin{align*}
    \Pr( \frac{|Z_i(v)|}{\sqrt{\nfem/2}} > t )&= \Pr( {|Z_i(v)|} > t\sqrt{\frac\nfem2} ) \le 
    2\exp\bigg(\dfrac{-t^2\nfem/4}{\specparaN \nfem +\frac t3\cdot\sqrt{\frac\nfem2}}\bigg).
\end{align*}
For $t\leq 3\specparaN \sqrt{\frac\nfem2}$, the above equation gives the following bound
\begin{align}
    \Pr( \frac{|Z_i(v)|}{\sqrt{\nfem/2}} > t )\le 
     2\exp\bigg(\dfrac{-t^2\nfem/4}{\specparaN\nfem +\specparaN\cdot{\frac\nfem2}}\bigg)=2\exp\Big(  \dfrac{-t^2}{6\specparaN}\Big).\label{eq:borror}
\end{align}
And, similarly
\begin{align}
\Pr( |\sqrt{Z_i(v)\specparaN}| > t )
&=\Pr( |{{Z_i(v)}}| > \frac{t^2}{\specparaN}) \leq 2\exp\bigg(\dfrac{-\frac{t^4}{2\specparaN^2}}{\specparaN\nfem +\frac{t^2}{3\specparaN}}\bigg)\nonumber.
\end{align}
For $t\geq 3\specparaN\sqrt{\frac\nfem2}$, the above equation give the following bound
\begin{align}
\Pr( |\sqrt{Z_i(v)\specparaN}| > t )&\le
2\exp\bigg(\dfrac{-\frac{t^4}{2\specparaN^2}}{\frac{2t^2}{9\specparaN} +\frac{t^2}{3\specparaN}}\bigg)
=
2\exp\bigg(  \dfrac{-t^2}{{10\specparaN}/{9}}\bigg)\nonumber.
\end{align}
Combining equation~\eqref{eq:borror} and the above equation we get
\begin{align*}
&\Pr( |{\xi_i}\widehat{Z}_i(v)| > t)  \le2\exp\Big(  \dfrac{-t^2}{6\specparaN}\Big).
\end{align*}

This shows that ${\xi_i}\widehat{Z}_i(v)$ is sub-Gaussian with sub-Gaussian norm $\le \sqrt{3\specparaN}$. Therefore, ${\xi_i}\widehat{Z}^2_i(v)$ is sub-exponential with sub-exponential norm  $\le 3\specparaN$. From Bernstein's inequality, for all $\varepsilon\ge1$,
\begin{equation*}
    \Pr\bigg(\sum_{i\in I} {\xi_i}\widehat{Z}^2_i(v) > \varepsilon \ell\specparaN \bigg)
    \leq
    2\exp^{-c\varepsilon \ell }.
\end{equation*}
Choosing $\varepsilon  = (14/c)\log (ek/\ell)$, bounds the above probability by
$(ek/\ell)^{-7\ell}$.
Taking the union bound over all possible $\ell, v, I$, and $J$, in above equation we get $\sum_{i\in I} {\xi_i}\widehat{Z}^2_i(v) \le \varepsilon \ell\specparaN  $ with probability {at least}
\begin{align}
     1-\sum_{\ell=1}^k 2^\ell {k\choose \ell}^2  \Paren{\frac{ek}{\ell}}^{-7\ell} \geq 1-{k^{-3}}.\label{eq:ex6}
\end{align}

\item $T_i=2:\ \nfex < \frac\nfem2$. \\
Note that,
\[
|{\bar\xi_i}\widehat{Z}_i(v)|\le |{\bar\xi_i}Z_i(v)| \overset{\text{(a)}}\le \frac 3 2 \nfem \specparaN
, 
\] 
where (a) uses equation~\eqref{eq:whenmed}. Then,
\begin{align}
 \sum_{i\in I} {\bar\xi_i}
\widehat{Z}_i(v)^2\le  \sum_{i\in [k]} {\bar\xi_i}
\widehat{Z}_i(v)^2\le \Paren{\frac 3 2}^2\sum_{i\in [k]} {\bar\xi_i}
\big({\nfem \specparaN}\big)^2
.\label{eq:postwed}
\end{align}
We bound the above term by showing:

\textbf{Claim: }With probability $\ge 1-2/k^3$,
\[
\sum_{i\in I} {\bar\xi_i}\big(\nfem
\big)^2\le  \cO\Big(\frac{\log^4 k +ke^{-\frac{\specparaN}{8}}}{\specparaN^2}\Big).
\]
\textbf{Proof of the claim:} Next, once again divide the rows into categories based on the expected count, and let $S_i$ denote the category of row $i$, 
\[
S_i := 
\begin{cases}
0 &  \text{ if } \nfem \le 2 \\
j  \quad\text{  for $j\ge 1$, }&\text{ if } {\nfem}\in (2^j,2^{j+1}].
\end{cases}
\]
For a given $M$, category $S_i$ of row $i$ is determined and is not a random variable unlike $T_i$.

Note that
\begin{align*}
\{\nfem> 2\nfex \}
&\equiv  \{\max\{1, \textstyle\frac{||M_{i,*}||_1}{\specparaN}\}> \max\{2, \textstyle\frac{2||X_{i,*}||_1}{\specparaN}\}\}
\equiv  \{||M_{i,*}||_1> \max\{2\specparaN, \textstyle 2||X_{i,*}||_1\}\}.
\end{align*}
Next,  
\begin{align*}
&\Pr\Big({2||X_{i,*}||_1} \leq {{||M_{i,*}||_1}}\Big)
=
\Pr\Big({||X_{i,*}||_1}-||M_{i,*}||_1 \leq- \frac{{||M_{i,*}||_1}}{2}\Big)\nonumber
\\
&=
\Pr\Big(\sum_{j\in [n]}(\mtrobsele-\mtrmnij) \leq -\frac{\sum_{j\in [n]}\mtrmnij}{2}\Big)
\leq e^{-\frac{||M_{i,*}||_1}{8}} 
,
\end{align*}
here we used $\E[X_{i,j}]=M_{i,j}$ and Chernoff bound.
Therefore,
\begin{align*}
\Pr[\nfem> 2\nfex ] \le 
\begin{cases}
0, &  \text{ if } ||M_{i,*}||_1 \le 2\specparaN, \\
e^{-\frac{||M_{i,*}||_1}{8}}= e^{-\frac{\specparaN\nfem}{8}}, &\text{ if } ||M_{i,*}||_1 > 2\specparaN.
\end{cases}   
\end{align*}

Then from the definitions of $S_i$ and $\nfem$, it follows that
\begin{align}\label{eq:exn1}
\Pr[{\bar\xi_i} = 1 ] = \Pr[\nfem> 2\nfex ] \le 
\begin{cases}
0 , &  \text{ if } S_i=0, \\
 e^{-\frac{\specparaN\nfem}{8}}\le e^{-\frac{2^{S_i}\specparaN}{8}}, &\text{ if } S_i \ge 1.
\end{cases}   
\end{align}


Using the Chernoff bound, for any $j\ge 1$, 
\begin{align}
    \Pr\bigg(\sum_{i: S_i=j} {\bar\xi_i}
    \geq
    15\log k +  2\sum_{i:S_i=j}\E[{\bar\xi_i}]\bigg)
    \leq k^{-5}.\label{eq:preprewed}
\end{align}
Let $\tau =\lfloor \log_2 \frac{16\ln k}{\specparaN}\rfloor$.  For a row $i$ such that $S_i > \tau $, using~\eqref{eq:exn1} we get
\[
\Pr[{\bar\xi_i} = 1]\leq k^{-4}.
\]
Therefore, with probability $\ge 1 -\frac{\{i:S_i > \tau\}}{k^5}\ge 1-\frac1{k^3} $
\begin{align}
\sum_{i: S_i>\tau} {\bar\xi_i} =0.\label{eq:prewed} 
\end{align}
Then
\begin{align*}
&\sum_{i\in [k]} {\bar\xi_i} \big(\nfem
\big)^2 \overset{\text{(a)}}= \sum_{j\geq 1} \,\sum_{i: S_i=j} {\bar\xi_i} \big(\nfem
\big)^2\\
&\overset{\text{(b)}}= \sum_{j=1}^\tau \, \sum_{i: S_i=j} \mathbbm{1}_{\{T_i=2\}} \big(\nfem
\big)^2\le \sum_{j=1}^\tau \,  \max_{i: S_i=j}\{\big(\nfem
\big)^2\}\sum_{i: S_i=j} {\bar\xi_i}\\
&\overset{\text{(c)}}\le\sum_{j=1}^\tau \,\max_{i: S_i=j}\{\big(\nfem
\big)^2\} \Big(15\log k +  2\sum_{i:S_i=j}\E[{\bar\xi_i}]\Big)\\
&\le 15 \tau  \log k \max_{i: S_i\le \tau}\{\big(\nfem
\big)^2\} + \sum_{j=1}^\tau \,\max_{i: S_i=j}\{\big(\nfem
\big)^2\} \max_{i: S_i=j}\{{\E[{\bar\xi_i}]}\}\Big(2|\{i:S_i=j\}|\Big)\\
&\overset{\text{(d)}}
\le 
15 \tau  \log k  \big(2^{\tau+1}
\big)^2
+ \sum_{j=1}^\tau \Big(|\{i:S_i=j\}|\Big)\times\Paren{2^{2j+2}\cdot  e^{-\frac{2^{j}\specparaN}{8}}} \\
&\le 
\cO(\frac{\log^4 k}{\specparaN^2}) + \cO\Big(\max_{j\ge 1}\Big\{\Big(2^{j}\specparaN\Big)^2 e^{-\frac{2^{j}\specparaN}{8}}\Big\}\cdot\specparaN^{-2}\cdot \sum_{j=1}^\tau  |\{i:S_i=j\}|\Big)
\\
&\overset{\text{(e)}}
\le \cO(\frac{\log^4 k}{\specparaN^2}) + \cO\Big(\max_{j\ge 1}\Big\{e^{-\frac{2^{j}\specparaN}{16}}\Big\}\cdot \frac{k}{\specparaN^2}\Big)\\
&\le
\cO\Big(\frac{\log^4 k +ke^{-\frac{\specparaN}{8}}}{\specparaN^2}\Big),
\end{align*}
with probability $1-1/k^3-\tau/k^5\ge 1-2/k^3$. Here (a) follows since $S_i=1$ implies ${\bar\xi_i}=0$, (b) follows from equation~\eqref{eq:prewed}, (c) uses equation~\eqref{eq:preprewed}, (d) follows from the definition of $S_i$ and equation~\eqref{eq:exn1}, and (e) follows as total number of rows are $k$ and $\frac{x^2\exp{(-x/8)}}{\exp{(-x/16)}}$ is bounded for $x>0$. This completes the proof of the claim.

Combining the Claim and ~\eqref{eq:postwed} gives the following bound on the contribution of rows in $T_i=2$,
\[
\sum_{i\in I} \bar{\xi_i} \widehat{Z}_i^2\le \cO({\log^4 k +ke^{-\frac{\specparaN}{8}}}),
\]
with probability $\ge 1-2/k^3$.

\end{enumerate}
Combining this bound and the bound in \eqref{eq:ex6},
\begin{align*}
    \sum_{i\in I}  \widehat{Z}_i^2 = \sum_{i\in I} \bar{\xi_i} \widehat{Z}_i^2+ \sum_{i\in I} {\xi_i} \widehat{Z}_i^2 \leq \cO(\log (ek/\ell)\ell\specparaN+\log^4 k +k \exp^{-\frac{\specparaN}{8}}),
\end{align*}
with probability $\ge 1-3k^{-3}$. Noting that $\ell\specparaN \geq \max(\log^4 k ,\ k \exp^{-\frac{\specparaN}{8}})$, completes the proof.
\end{proof}



The next lemma combines the bound obtained on $\ell_{\infty}\rightarrow\ell_2$ norm in the above lemma and Grothendieck-Piesch factorization to obtain bound on the spectral norm.

\begin{lemma}\label{lem:con}
With probability $\ge 1-3k^{-3}$, for
any $ \max(\frac{\log^4 k}{\specparaN} ,\  \frac{k \exp^{-\frac{\specparaN}{8}}}{\specparaN})\leq \ell\leq k$, and any $I,J\subseteq [k]$ of size $|I|,|J| =  \ell$, there exists a subset $J'\subseteq J$ such that $\sum_{j\in J'}\nbex\le \ell/2$ and  
\[
    ||{\big(R(\noisemtrspt,\wgenX)\big)}_{I\times J' }||\leq c\sqrt{\specparaN\log(ek/\ell)}.
\]
\end{lemma}
\begin{proof}
Applying Grothendieck-Piesch factorization in Theorem~\ref{th:groth} on matrix $(\nf{\noisemtrspt})_{I\times J} $, implies that there is a vector $\mu=(\mu(1),...,\mu(m))$ with $\mu(j)\ge 0$ and $\sum_j\mu(j)=1$ such that
\[
||(\nf\cdot {\noisemtrspt})_{I\times J} \cdot D^{-1/2}(\mu)||\le \sqrt{\frac\pi2}||(\nf\cdot {\noisemtrspt})_{I\times J}||_{\infty\rightarrow 2}
\]
Then,
\begin{align*}
||\big(\nf\cdot {\noisemtrspt}\big)_{I\times J}\cdot D^{-\frac{1}2}(\mu)||
&=
||\big(\nf\cdot {\noisemtrspt}\big)_{I\times J}\cdot D^{-\frac{1}2}(\mu)\cdot\nb\cdot  D^{\frac12}(\nbx)||\\
&=
||\big(\nf\cdot {\noisemtrspt}\cdot \nb\cdot D^{\frac12}\Big(\nbx\circ \frac1\mu\Big)\big)_{I\times J} ||\\
&=
||\big(R(\noisemtrspt,\wgenX)\cdot D^{\frac12}\Big(\nbx\circ \frac1\mu\Big)\big)_{I\times J} ||,
\end{align*}
here $\nbx\circ \frac1\mu := (\frac{\nbx(1)}{\mu(1)},...,\frac{\nbx(m)}{\mu(m)})$. 
Let 
\[
J' := \{j\in J: \frac{\nbex}{\mu(j)}\ge \frac{\ell}{2} \}
\]
and $\bar J = J\setminus J'$. Then
\begin{align*}
||\big(R(\noisemtrspt,\wgenX)\cdot D^{\frac12}\Big(\nbx\circ \frac1\mu\Big)\big)_{I\times J} ||\ge ||\big(R(\noisemtrspt,\wgenX)\cdot D^{\frac12}\Big(\nbx\circ \frac1\mu\Big)\big)_{I\times J'} ||
&\ge \sqrt{\frac{\ell}{2}}||\big(R(\noisemtrspt,\wgenX)\big)_{I\times J'}, ||
\end{align*}
here the last step follows from the definition of $J'$.
Therefore,
\begin{align*}
||\big(R(\noisemtrspt,\wgenX)\big)_{I\times J'} ||\le \sqrt{\frac\pi\ell}||(\nf\cdot {\noisemtrspt})_{I\times J}||_{\infty\rightarrow 2}\le c\sqrt{\specparaN\log(ek/\ell)}.
\end{align*}
Next, we bound the weight of the columns that are excluded from $J'$.
\[
\sum_{j\in \bar J} \nbex \le \frac{\ell}{2}\sum_{j\in \bar J}\mu(j) \le \frac{\ell}{2}\sum_{j\in [k]}\mu(j)=\frac{\ell}{2}.
\]
\end{proof}
Applying the above lemma on $\big(R(\noisemtrspt,\wgenX)\big)^\intercal$, in place of $\big(R(\noisemtrspt,\wgenX)\big))$, from the symmetry we get: 
\begin{lemma}\label{lem:con2}
With probability $\ge 1-3k^{-3}$, for
any $ \max(\frac{\log^4 k}{\specparaN} ,\  \frac{k \exp^{-\frac{\specparaN}{8}}}{\specparaN})\leq \ell\leq k$, and any $I,J\subseteq [k]$ of size $|I|,|J| =  \ell$, there exists a subset $I'\subseteq I$ such that $\sum_{i\in I'}\nfex\le \ell/2$ and  
\[
    ||{\big(R(\noisemtrspt,\wgenX)\big)}_{I'\times J}||\leq c\sqrt{\specparaN\log(ek/\ell)}.
\]
\end{lemma}
The next lemma bounds the norm of a matrix using the norm of its sub-matrices. Incorporating the above bound on the norm of submatrices, this will complete the proof of theorem~\ref{th:sp norm}.
\begin{lemma}\label{le:subnorm}
Let $\gnrmtr\in \Rkm $ and $I_1,I_2,I_3,...,I_t$ be $t$ disjoint subsets of $[k]$ such that $\cup_{j=1}^t  I_j = [k]$. Then
$||\gnrmtr || \leq \sqrt{\sum_{j=1}^t ||\gnrmtr _{I_j}||^2}$.
\end{lemma}
Proof of the above lemma is given in Appendix~\ref{app:LA}.

\textit{Proof of theorem~\ref{th:sp norm gen}:} Main component in the proof is Lemma~\ref{lem:con} and Lemma~\ref{lem:con2}. We need to apply these lemmas in multiple rounds. 

For round $j$, we apply these Lemmas for some, $\ell = \ell_j$ and $I=I_j$, and $J =J_j$ such that $|I_j|,|J_j|\le k/2^{j-1}$, where $\ell_j$, $I_j$ and $J_j$ are defined later. 

First applying Lemma~\ref{lem:con2} we get, a subset $I'_j\subseteq I_j $ such that,
\begin{equation}
    ||{\big(R(\noisemtrspt,\wgenX)\big)}_{I'_j\times J_j }||\leq \cO(\sqrt{j\cdot\specparaN }), \label{eq:b}
\end{equation}
and the weight of the excluded rows $I_j\setminus I'_j$ is at most $\sum_{i\in I_j\setminus I'_j}\nfex\le k/2^{j}$. Since weight of each row is at-least 1, this implies that the number of rows excluded in round $j$ are also at most $|I_j\setminus I'_j|\le k/2^{j}$. 

Similarly, applying Lemma~\ref{lem:con} we get, a subset $J'_j\subseteq J_j $ such that $\sum_{i\in J_j\setminus J'_j}\nbx(i)\le k/2^{j}$,
\begin{equation}
    ||{\big(R(\noisemtrspt,\wgenX)\big)}_{I_j\times J'_j }||\leq \cO(\sqrt{j\cdot\specparaN }). \label{eq:a}
\end{equation}
and $\sum_{i\in J_j\setminus J'_j}\nbx(i)\le k/2^{j}$ and $|J_j\setminus J'_j|\le k/2^{j}$. 
Since zeroing out rows from a matrix reduces the spectral norm, the above equation gives
\begin{equation}
    ||{\big(R(\noisemtrspt,\wgenX)\big)}_{(I_j\setminus I'_j)\times J'_j }||\leq \cO(\sqrt{j\cdot\specparaN }). \label{eq:c}
\end{equation}

In round $j=1$, we start with $\ell_1=k$ and $I_1=J_1=[k]$. 
For round $j>1$ we chose, $\ell_j :=\frac{\ell_{j-1}}{2}$, $I_{j} := I_{j-1}\setminus I'_{j-1}$ and $J_j := J_{j-1}\setminus J'_{j-1}$. 
Note that $I_{j}$ and $J_{j}$ are the excluded rows and columns in the concentration bounds of the previous round. 

We use this procedure for $t = \lceil\log (k/\epsilon k)\rceil$ rounds, so that the weight and the number of excluded rows, and columns, in the end is at-most $\epsilon k$.

Let $\mathcal M_j := I_j\times J_j$, which is of size $k/2^{j-1}\times k/2^{j-1} $, $\mathcal{R}_j := I'_j\times J_j $ and $\mathcal{C}_j := I_j\setminus I'_j\times J'_j = I_{j+1}\times J'_j $. Note that equation~\ref{eq:b} and~\ref{eq:c} gives the concentration bound for sub-matrices corresponding to $\mathcal{R}_j$ and $\mathcal{C}_j$, respectively.
Figure \ref{figure1} shows this construction.
\begin{figure}[H]
  \centering
    \includegraphics[height=8cm, width=8cm]{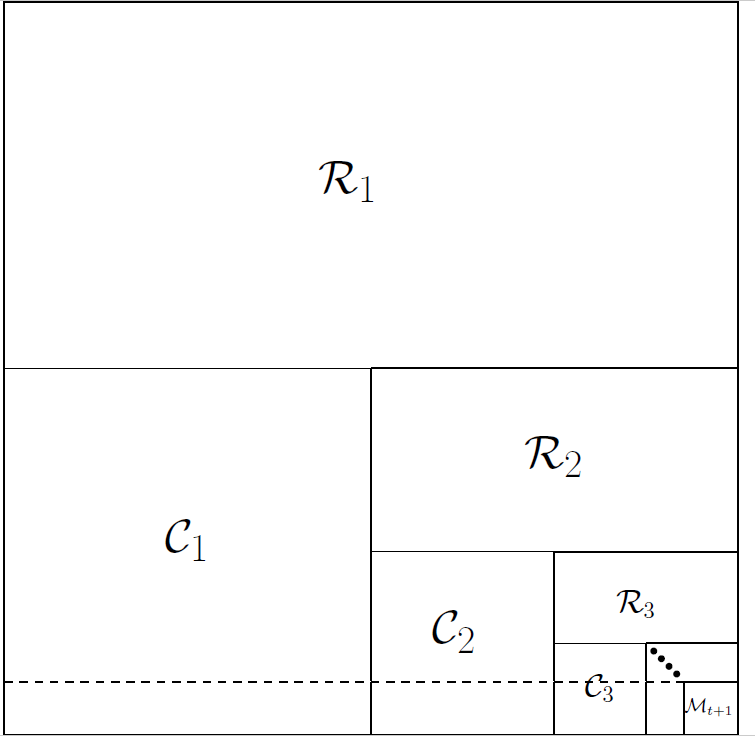}
    \caption{Construction of submatrices in proof of theorem~\ref{th:sp norm}.}\label{figure1}
\end{figure}
This construction decomposes the submatrix indexed by $\cM_j$ into three submatrices
\begin{align*}
  {\big(R(\noisemtrspt,\wgenX)\big)} _{\mathcal{M}_j} =  &{\big(R(\noisemtrspt,\wgenX)\big)} _{\mathcal{R}_j}+{\big(R(\noisemtrspt,\wgenX)\big)} _{\mathcal{C}_j}+{\big(R(\noisemtrspt,\wgenX)\big)} _{\mathcal{M}_{j+1}}.
\end{align*}
Applying the above equation recursively,
\begin{align*}
  &{\big(R(\noisemtrspt,\wgenX)\big)}  = {\big(R(\noisemtrspt,\wgenX)\big)} _{\mathcal{M}_1}\\ 
   &= \sum_{j=1}^t  {\big(R(\noisemtrspt,\wgenX)\big)} _{\mathcal{R}_j}+  \sum_{j=1}^t {\big(R(\noisemtrspt,\wgenX)\big)} _{\mathcal{C}_j}
  +{\big(R(\noisemtrspt,\wgenX)\big)} _{\mathcal{M}_{t+1}}\\
  &= {\big(R(\noisemtrspt,\wgenX)\big)} _{\cup_{j=1}^t  \mathcal{R}_j}+  {\big(R(\noisemtrspt,\wgenX)\big)} _{\cup_{j=1}^t \mathcal{C}_j}
  +{\big(R(\noisemtrspt,\wgenX)\big)} _{\mathcal{M}_{t+1}},
\end{align*}
where the last equality follows as $\mathcal{R}_j$'s and $\mathcal{C}_j$'s are disjoint.
Then
\begin{align*}
  ||{\big(R(\noisemtrspt,\wgenX)\big)}_{([k]\times [k])\setminus \mathcal{M}_{t+1}}||&= 
  ||R(\noisemtrspt,\wgenX) -{\big(R(\noisemtrspt,\wgenX)\big)}_{\mathcal{M}_{t+1}}||\\
  &=||{\big(R(\noisemtrspt,\wgenX)\big)} _{\cup_{j=1}^t  \mathcal{R}_j}+  {\big(R(\noisemtrspt,\wgenX)\big)} _{\cup_{j=1}^t \mathcal{C}_j} ||\\
  &\overset{\text{(a)}}\leq  ||{\big(R(\noisemtrspt,\wgenX)\big)} _{\cup_{j=1}^t  \mathcal{R}_j}|| + || {\big(R(\noisemtrspt,\wgenX)\big)} _{\cup_{j=1}^t \mathcal{C}_j} ||\\
  &\overset{\text{(b)}}\leq  \textstyle\sqrt{\sum_{j=1}^t {\big(R(\noisemtrspt,\wgenX)\big)} _{\mathcal{R}_j} }+\sqrt{\sum_{j=1}^t {\big(R(\noisemtrspt,\wgenX)\big)} _{\mathcal{C}_j} }\nonumber\\
  &\overset{\text{(c)}}\leq \textstyle2\cdot \cO(\sqrt{\sum_{j=1}^t  j \cdot \specparaN })\leq   \cO(t \sqrt{\specparaN }) = \cO( \log(k/\epsilon k)  \sqrt{\specparaN }),
\end{align*}
where (a) follows from triangle inequality, (b) follows from Lemma~\ref{le:subnorm}, and (c) follows from inequalities~\eqref{eq:b} and~\eqref{eq:c}. 
Note that $\sum_{i\in I_{t+1}}\nfex =\sum_{i\in I_t\setminus I'_t}\nfex \le k/2^{t}\le \epsilon k$ and, similarly, $\sum_{i\in J_{t+1}}\nbx(i)\le \epsilon k$.

Recall that $\mathcal{M}_{t+1} = I_{t+1}\times J_{t+1}$, therefore zeroing out rows $I_{t+1}$ from $([k]\times [k])\setminus \mathcal{M}_{t+1}$ results in $([k]\setminus I_{t+1}\times [k])$. And since zeroing out rows reduces of a matrix reduces the spectral norm, from the above equation we get
\begin{align*}
||{\big(R(\noisemtrspt,\wgenX)\big)}_{([k]\setminus I_{t+1}\times [k])}||\le  ||{\big(R(\noisemtrspt,\wgenX)\big)}_{([k]\times [k])\setminus \mathcal{M}_{t+1}}||\le \cO( \log(k/\epsilon k)  \sqrt{\specparaN }),
\end{align*}
where $\sum_{i\in I_{t+1}}\nfex =\sum_{i\in I_t\setminus I'_t}\nfex \le k/2^{t}\le \epsilon k$. Letting $I_{t+1}$ to be the set of contaminated rows completes the proof of the theorem.
\hfill $\qedsymbol$

In the next subsection, we derive a useful implication of Lemma~\ref{inf2}.  
\subsection{Proof of Lemma~\ref{lem:absdif}}\label{lem:seccon}
The lemma upper bounds the sum of the absolute difference between the expected and observed samples in each row of $X$. We restate the lemma.
\begin{lemma*}
With probability $\ge 1-3k^{-3}$,
\[
\sum_{i}{|\sum_j N_{i,j}|}=\cO(k\sqrt{\navg}).
\]
\end{lemma*}
\begin{proof}
To prove the above Lemma we use Lemma~\ref{inf2}, for $\specparaN=\navg$, $\ell=k$ and $I=J=[k]$. The lemma implies that w.p. $\ge 1-3/k^3$,
\[
\max_{v\in\{-1,1\}^k} ||\snf\cdot {\noisemtrspt}\cdot v|| = ||\snf\cdot {\noisemtrspt}||_{\infty\rightarrow 2}\leq \cO(\sqrt{\navg k}).
\]
From the definition of $\ell_\infty\rightarrow \ell_2$ norm,
\[
\max_{v\in\{-1,1\}^k} ||\snf\cdot {\noisemtrspt}\cdot v|| = ||\snf\cdot {\noisemtrspt}||_{\infty\rightarrow 2}.
\]
In the above equation choosing $v=(1,1,...,1)$ and taking the square on the both side, we get
\[
\sum_{i}\Big(\sum_{j}\frac{\elenoisemtrspt}{\sqrt{\snfex}}\Big)^2\le ||\snf\cdot {\noisemtrspt}||^2_{\infty\rightarrow 2}
 \leq \cO({\navg k}).
\]
The above equation can be rewritten as,
\[
\sum_{i}\frac{|\sum_j N_{i,j}|^2}{{\snfex}}
 \leq \cO({\navg k}).
\]
Then
\[
\sum_{i}{|\sum_j N_{i,j}|}=\sum_{i}\frac{|\sum_j N_{i,j}|}{\sqrt{\snfex}}\cdot\sqrt{\snfex} \le \Big(\sum_{i}\frac{|\sum_j N_{i,j}|^2}{{\snfex}}\Big)^{\frac12} \cdot\Big(\sum_i{\snfex}\Big)^{\frac12}=\cO(k\sqrt{\navg}).
\]
here we used the Cauchy-Schwarz inequality, the previous equation, and Lemma~\ref{lem18} that states $\sum_i{\snfex}\le 2k$. 
\end{proof}

\section{Counterexample}
\label{app: counter}
The authors of~\cite{le2017concentration}
posed the following question, whose affirmative answer may have simplified low-rank matrix recovery. They posed the question for Bernoulli-parameter matrix.

Let $\mtrmn\in \Rkk$ be a Bernoulli-parameter matrix where $\forall i,j\in [k]$,  $||M||_{i,*},||M||_{*,j} \le n_{\text{max}}$, for some $n_{\text{max}}$.

Let $\mtrobs=[\mtrobsele]$, where $\mtrobsele \sim \Ber(\mtrmnij)$, be the observation matrix of $\mtrmn$, and let $\mtrobs_0$ be the matrix obtained by zeroing-out the rows and columns of $\mtrobs$ whose total count is $> 2n_{\text{max}}$. Does $\mtrobs_0$ converges to $\mtrmn$ \whp as
\begin{align*}
    \norm{}{\mtrobs_0-\mtrmn} =\cO(\sqrt{n_{\text{max}}})?
\end{align*}

Unfortunately, the following counterexample answers this question in negative.
Therefore additional work, such as presented in this paper, is needed to recover low-rank matrices.

For any $k$, choose any $n_{\text{max}}\le \sqrt{\log k}/8$ that grows with $k$, 
and consider the block diagonal matrix
\[
\mtrmn = \left( \begin{array}{ccccc}
B_1 &  &  &  &  \\
 & B_2 &  &  &  \\
 &  & \ddots & &  \\
 &  &  & B_{(\frac k{2n_{\text{max}}}-1)} &  \\
 &  &  &  & B_{\frac k{2n_{\text{max}}}} \\
\end{array} \right)
\]
consisting of $k/(2n_{\text{max}})$ (for simplicity assume it is an integer) identical blocks $B_i=B$, each a submatrix of size $2n_{\text{max}}\times 2n_{\text{max}}$ whose entries are all $1/2$.
Except for the blocks $B_i$'s, all the other entries of $\mtrmn$ are zero. Then $\mtrmn$ satisfies $\forall i, j\in [n]$, $||M||_{i,*} = ||M||_{*,j}  = n_{\text{max}}$. 

Note that, for the above matrix $\navg = ||M||_1/k = n_{\text{max}}$.

The observation matrix of $\mtrmn$ is
\[
\mtrobs = \left( \begin{array}{ccccc}
 \hat B_1 &  &  &  &  \\
 &  \hat B_2 &  &  &  \\
 &  & \ddots & &  \\
 &  &  &  \hat B_{(\frac k{2n_{\text{max}}}-1)} &  \\
 &  &  &  &  \hat B_{\frac k{2n_{\text{max}}}} \\
\end{array} \right).
\]
Note that $\mtrobs$ has non-zero entries only in locations corresponding to the diagonal blocks $B$. Also,  $ \forall i, j\in [k],\ ||X||_{i,*} , ||X||_{*,j} < 2n_{\text{max}}$. 
Therefore zeroing out rows and columns of $\mtrobs$ with more than $n_{\text{max}}$ ones would not affect it and $X_0=X$. 
From Theorem \ref{th:interlace},
\[
\norm{}{\difmat} \geq \max_i \norm{}{\hat B_i -  B_i}.
\]

Since $ \hat B_i$ is the observation matrix for the
$2n_{\text{max}}\times 2n_{\text{max}}$ block 
$B_i$ whose entries are all $1/2$, and $n_{\text{max}} = o(\sqrt{\log k})$,
\[
\Pr\Paren{\hat B_i=0}
=
(1/2)^{4n_{\text{max}}^2}
>
1/\sqrt k.
\]

The probability that the whole $2n_{\text{max}}\times 2n_{\text{max}}$ block $\hat B_i$ is 0, is $ (1/2)^{4n_{\text{max}}^2}$, which since $n_{\text{max}} = o(\sqrt{\log k})$, is
$> 1/\sqrt k $.
Hence \whp, at least one of the block $j\in [\frac k{2n_{\text{max}}}]$ in $\mtrobs$ is zero.
Hence \whp,
\[
\norm{}{\difmat} \geq \norm{}{B_j } = n_{\text{max}} \gg \Omega(\sqrt{n_{\text{max}}}).
\]
This counterexample answers the question raised by the authors of~\cite{le2017concentration} in negative. 

We note that the same counterexample works for the regularization $R(X-M,\swgenX)$ used in this paper. Because if block $\hat B_j$ of $X$ is zero, from the definition of $\swgenX$, it is easy to see that for the rows and columns corresponding to the block $\hat B_j$, the regularization weights are 1, which implies that 
\[
\norm{}{R(X-M,\swgenX)} \geq \norm{}{B_j } = n_{\text{max}} ,
\]
extending the counterexample for the regularization $R(X-M,\swgenX)$ as well.

\section{Linear Algebra Proofs}
\label{app:LA}
\subsection{Proof of Lemma~\ref{th:recovery}}
\begin{lemma*}
For any rank-$r$ matrix $\gnrmtr\in\Rkkr$, matrix $B\in\Rkk$, and weights $w$,
\begin{align*}
    \norm1{A - B^{(r,w)}}
    \leq
    \textstyle\sqrt{r\cdot\paren{{\textstyle\sum}_{i}  w^\front(i)}(\sum_{j}  w^\back(j))}\cdot \norm{}{R(A-B,w)}.
\end{align*}
\end{lemma*}
\begin{proof}
Recall that 
\[
B^{(r,w)} := D^{\frac12}(w^\front)\cdot
R(B,w)^{(r)}
\cdot D^{\frac12}(w^\back),
\]
where $R(B,w)= D^{-\frac12}(w^\front)\cdot B\cdot D^{-\frac12}(w^\back)$ is regularized matrix $B$ and $R(B,w)^{(r)}$ is its rank $r$-truncated SVD. 

We first upper bound the spectral norm of $R(A,w)-R(B,w)^{(r)}$ in terms of the spectral norm of $R(A-B,w)$.
By Weyl's Inequality~\ref{th:weyl}, and the rank $r$ of $\gnrmtr$, 
\[
\sigma_{r+1}(R(B,w)) \leq \sigma_{r+1}(R(A,w)) +\norm{}{R(A,w)-R(B,w)} = \norm{}{R(A-B,w)}.
\]
Hence by the triangle inequality and the salient property of truncated SVD's,
\begin{align*}
\norm{}{R(A,w)-R(B,w)^{(r)}}
&\le
\norm{}{R(B,w)^{(r)}-R(B,w)}+\norm{}{R(A,w)-R(B,w)} \\
 &=
 \sigma_{r+1}(R(B,w))+\norm{}{R(A-B,w)} 
 \le
 2\norm{}{R(A-B,w)}.
\end{align*}
Since $\gnrmtr-\gnrmtrotrtrSVD$ is the difference of two rank-$r$ matrices, it has rank $\le 2r$.
Then applying Lemma~\ref{lem:spll} for matrix $(R(A,w)-R(B,w)^{(r)})$, and noting that $A= D^{\frac12}(w^\front)\cdot
R(A,w)
\cdot D^{\frac12}(w^\back)$ completes the proof.
\end{proof}

\subsection{Proof of Lemma~\ref{dampcon}}
\begin{lemma*}
For any matrix $\gnrmtr$ and weight vectors $w^\front$ and $w^\back$ with positive entries
\[
||D^{-\frac12}(w^\front)\cdot\gnrmtr\cdot D^{-\frac12}(w^\back)|| \leq  \sqrt{\max_{i}\frac{||\gnrmtr_{i,*}||_1}{w^\front(i)} \times \max_{j}\frac{||\gnrmtr_{*,j}||_1}{w^\back(j)}}.
\]
\end{lemma*}
\begin{proof}
For a unit vector $v=(v(1)\upto v(m))\in R^m$,
\begin{align*}
  &||D^{-\frac12}(w^\front)\cdot\gnrmtr\cdot D^{-\frac12}(w^\back)\cdot v||^2 = \sum_i\Big( \sum_j \frac{\gnrmtr_{i,j} v(j)}{\sqrt{w^\front(i)\cdot w^\back(j)}}\Big)^2\\
  &\le \sum_i\Big( \sum_j \frac{|\gnrmtr_{i,j}|| v(j)|}{\sqrt{w^\front(i)\cdot w^\back(j)}} \Big)^2\\
  &= \sum_i\Big( \sum_j \sqrt{\frac{|\gnrmtr_{i,j}|}{w^\front(i)}}\sqrt{\frac{|\gnrmtr_{i,j}|}{w^\back(j)}}| v(j)| \Big)^2\\
  &\overset{\text{(a)}}\leq \sum_i\Bigg(\Big( \sum_j \frac{|\gnrmtr_{i,j}|}{w^\front(i)} \Big)\Big( \sum_j  \frac{|\gnrmtr_{i,j}|}{w^\back(j)} v(j)^2 \Big)\Bigg)
  \leq \max_{i'}\frac{||\gnrmtr_{i',*}||_1}{w^\front(i')}  \sum_i  \sum_j  \Big(\frac{|\gnrmtr_{i,j}|}{w^\back(j)} v(j)^2 \Big)\\
  &= \max_{i'}\frac{||\gnrmtr_{i',*}||_1}{w^\front(i')} \sum_j \Big(v(j)^2 \sum_i  \frac{|\gnrmtr_{i,j}|}{w^\back(j)}\Big)\le  \max_{i'}\frac{||\gnrmtr_{i',*}||_1}{w^\front(i')}\sum_j v(j)^2 \frac{||\gnrmtr_{*,j}||_1}{w^\back(j)} \\
  &\le  \max_{i'}\frac{||\gnrmtr_{i',*}||_1}{w^\front(i')} \times \max_{j'}\frac{||\gnrmtr_{*,j'}||_1}{w^\back(j')} \sum_j v(j)^2  =  \max_{i'}\frac{||\gnrmtr_{i',*}||_1}{w^\front(i')} \times \max_{j'}\frac{||\gnrmtr_{*,j'}||_1}{w^\back(j')},
\end{align*}
where (a) uses the Cauchy-Schwarz inequality. Observing that above is true for arbitrary unit vector $v$ completes the proof.
\end{proof}

\subsection{Proof of Lemma ~\ref{lem:number}}
\begin{lemma*}
Let $\gnrmtr $ be an $k\times m$ matrix such that $\sigma_{1}(\gnrmtr )\leq \alpha $ and $\sigma_{r+1}(\gnrmtr )\leq \beta$. Then the number of disjoint row subsets ${I}\subset [k] $ such that $||\gnrmtr _{{I}}||> 2\beta$ is at most $\Big(\frac{r\alpha}{\beta}\Big)^2 $.
\end{lemma*}
\begin{proof}
Let $\gnrmtr  = \sum_{i=1}^{\min\{k,m\}} \sigma_i(\gnrmtr )u_i {v_i}^\intercal $ be the SVD decomposition of $\gnrmtr $. Recall that $A^{(r)}= \sum_{i=1}^{r} \sigma_i(\gnrmtr )u_i {v_i}^\intercal $ and let $B=\gnrmtr -A^{(r)}$.\\
Note that $||A^{(r)}|| = \sigma_{1}(\gnrmtr )\leq \alpha$ and the matrix $A^{(r)}$ has rank $r$ i.e. $\sigma_{r+1}(A^{(r)}) =0$. And  $||B|| =\sigma_{1}(B) = \sigma_{r+1}(\gnrmtr )\leq \beta$.\\
To prove the lemma we upper bound the number of disjoint subsets $I\subset [k]$ such that $||\gnrmtr _{I}||> 2\beta$. Let $I\subset [k]$ be one such subset such that $||\gnrmtr _{I}||> 2\beta$. Then 
\begin{align*}
 ||\gnrmtr _{I}|| \overset{\text{(a)}}\leq ||A^{(r)}_{I}|| +||B_{I}||  \overset{\text{(b)}}\leq ||A^{(r)}_{I}|| +||B|| \leq ||A^{(r)}_{I}||+\beta,  
\end{align*}
where inequality (a) follows from the triangle inequality and (b) follows from Theorem~\ref{th:interlace}.
Hence,
\[
||A^{(r)}_{I}||\geq \beta.
\]
Note that since row span of $A^{(r)}$, and hence $A^{(r)}_{I}$ is $span\{v_1,v_2,...,v_r\}$, therefore there exists a unit vector, $ v = \sum_{i=1}^r a_i v_i$ (here $\sum_{i=1}^r a_i^2 =1 $, since $v$ is a unit vector), such that $||A^{(r)}_{I} v || \geq \beta $. Therefore, 
\begin{align}
    \beta \leq ||A^{(r)}_{I} v || = ||A^{(r)}_{I} \sum_{i=1}^r a_i v_i || \leq \sum_{i=1}^r |a_i|\,||A^{(r)}_{I}  v_i ||\leq  \sum_{i=1}^r ||A^{(r)}_{I}  v_i || .\label{eq:ex11}
\end{align}

Let $I_1,I_2,....,I_t$ be the $t$ disjoint blocks such that $||\gnrmtr _{I_j}||> 2\beta,\ \forall j \in [t]$. Next,  
\begin{align*}
    \sum_{i=1}^r ||A^{(r)}v_i||&\geq \sum_{i=1}^r ||A^{(r)}_{\cup _{j=1}^t  I_j}v_i|| = \sum_{i=1}^r ||\sum_{j=1}^t  A^{(r)}_{I_j}  v_i ||\\
    & \overset{\text{(a)}}= \sum_{i=1}^r\sqrt{\sum_{j=1}^t  || A^{(r)}_{I_j}  v_i ||^2}  \overset{\text{(b)}}\geq \sum_{i=1}^r \frac{\sum_{j=1}^t  || A^{(r)}_{I_j}  v_i ||}{\sqrt{t}}  \overset{\text{(c)}}\geq {\sqrt{t}}\beta .
\end{align*}
Here equality (a) follows since $A^{(r)}_{I_j} v_i$'s for $j\in [t]$ and fixed $i$ are orthogonal. Inequality (b) follows from the AM-GM inequality (c) follows from \eqref{eq:ex11}.\\
We also have $ \sum_{i=1}^r ||A^{(r)}v_i|| =\sum_{i=1}^r \sigma_i(\gnrmtr ) \leq r\sigma_1(\gnrmtr ) \leq r\alpha $. Therefore we get,
$t\leq \Big(\frac{r\alpha}{\beta}\Big)^2$.
\end{proof}

\subsection{Proof of Lemma~\ref{lem:bad part}}
\begin{lemma*}
Let $\gnrmtr = B+C$ and $\gnrmtr  = \sum_{i} \sigma_i(\gnrmtr )u_i {v_i}^\intercal $ be the SVD decomposition of $\gnrmtr $. And $\sigma_{r+1}(B) \leq  \beta $ and $ ||C v_i||\leq 2\beta $ for $i\in [2r]$. Then
$\sigma_{2r}(\gnrmtr ) \leq 4\beta $.
\end{lemma*}
\begin{proof}
Let $\gnrmtr^{(2r)} = \sum_{i=1}^{2r} \sigma_i(\gnrmtr )u_i {v_i}^\intercal $ be rank $2r$ truncated SVD of $\gnrmtr$, then
\begin{align}
  \sigma_{i}(\gnrmtr^{(2r)} )=  \sigma_{i}(\gnrmtr ),\ \forall\ i \leq 2r \label{eq:asd}
\end{align}
and
\begin{align}
    \gnrmtr^{(2r)} = \sum_{i=1}^{2r} \gnrmtr v_i {v_i}^\intercal = \sum_{i=1}^{2r} (B+C)v_i {v_i}^\intercal  =\hat{B}+\hat C.\label{eq:asd1}
\end{align}
Here $\hat B = \sum_{i=1}^{2r} Bv_i {v_i}^\intercal $ and $\hat C = \sum_{i=1}^{2r} Cv_i {v_i}^\intercal $.

Since $v_i$'s are orthogonal unit vector, $\sum_{i=1}^{2r} v_i {v_i}^\intercal $ is a projection matrix for subspace $S = span\{v_1,v_2,..,v_{2r}\}$. And for any Projection matrix $P$ we have, $||Pu|| \leq ||u||$. Therefore,
\begin{align}\label{eq:qqq}
    ||\hat B^\intercal  u|| = || (\sum_{i=1}^{2r} v_i {v_i}^\intercal ) B^\intercal u || \leq ||B^\intercal u||.
\end{align}
Next, using Courant-Fischer theorem, $\forall \ i \leq \min\{k,m\}$, there exists a subspace $S_i^*$ with dimension $\ dim(S_i^*) = i$, such that
\begin{align*}
 \sigma_{i}(\hat B^\intercal )
 &=
 \min_{u\in S_i^*, ||u|| =1} ||\hat B^\intercal  u|| \\
 &\overset{\text{(a)}}\leq
 \min_{u\in S_i^*, ||u|| =1} || B^\intercal  u||\\
&\leq
 \max_{S:dim(S) = i }\ \min_{u\in S, ||u|| =1}|| B^\intercal  u||\\ 
 &\overset{\text{(b)}}=
 \sigma_{i}(B^\intercal ) =  \sigma_{i}(B) ,
\end{align*}
where inequality (a) uses \eqref{eq:qqq} and (b) again from Courant-Fischer theorem. 
Therefore, 
\begin{align}
  \sigma_{i}(B)\geq  \sigma_{i}(\hat B),\ \forall\ i \leq {\min\{k,m\}}.\label{eq:asd2}
\end{align}
Using \eqref{eq:asd}, \eqref{eq:asd1}, Weyl's inequality \ref{th:weyl} and \eqref{eq:asd2}:
\begin{align}
  \sigma_{2r}(\gnrmtr ) = \sigma_{2r}(\gnrmtr^{(2r)} ) \le \sigma_{r+1}(\hat B)+\sigma_{r}(\hat C)
  \le \sigma_{r+1}(B)+\sigma_{r}(\hat C) 
  \le \beta+\sigma_{r}(\hat C).\label{eq:qwe}
\end{align}
Now 
\begin{align*}
    \hat{C} \hat{C}^\intercal  = \sum_{i=1}^{2r} (Cv_i)(Cv_i)^\intercal =  \sum_{i=1}^{2r} \hat{u_j}\hat{u_j}^\intercal . 
\end{align*}
Here $\hat{u_j}= Cv_i $, hence $||\hat{u_j}||\leq 2\beta $. Note that $\hat{C} \hat{C}^\intercal $ and $\hat{u_j}\hat{u_j}^\intercal $'s are Hermitian matrices. Let $\lambda_i(.)$ denotes the $i^{th}$ largest eigenvalue of the matrix. Then 
\[ \lambda_i(\hat C \hat{C}^\intercal ) =\sigma_i^2(\hat{C}).   \]
For rank-1 matrices $\hat{u_j}\hat{u_j}^\intercal $,
\[
\lambda_1(\hat{u_j}\hat{u_j}^\intercal ) = ||\hat{u_j}||^2 \leq 4\beta^2 \quad \text{and}\quad  \lambda_i(\hat{u_j}\hat{u_j}^\intercal ) = 0,\ \forall j\in [2r],\quad i\geq 2. 
\]
Then using Lidskii's theorem~\cite{bhatia2013matrix}, leads to
\begin{align*}
   \sum_{i=1}^{r}\lambda_i\big(\sum_{i=1}^{2r} \hat{u_j}\hat{u_j}^\intercal \big) &\leq  \sum_{i=1}^{r}\lambda_i\big(\sum_{i=1}^{2r-1} \hat{u_j}\hat{u_j}^\intercal \big)+ \sum_i^{r}\lambda_i\big( \hat{u_{2r}}\hat{u_{2r}}^\intercal \big)\\
   &\leq  \sum_{i=1}^{r}\lambda_i\big(\sum_{i=1}^{2r-1} \hat{u_j}\hat{u_j}^\intercal \big)+4\beta^2.
\end{align*}
By repeated application of Lidskii's theorem, we get
\begin{align*}
   \sum_{i=1}^{r}\lambda_i\big(\sum_{i=1}^{2r} \hat{u_j}\hat{u_j}^\intercal \big) \leq  8r\beta^2.
\end{align*}
Since $\lambda_i$'s are decreasing, it follows
\begin{align}
   r\lambda_r\big(\sum_{i=1}^{2r} \hat{u_j}\hat{u_j}^\intercal \big) \leq  8r\beta^2\ \Rightarrow \lambda_r(\hat C \hat{C}^\intercal ) =\sigma_r^2(\hat{C}) \leq 8\beta^2.\label{eq:qwe2}
\end{align}
Combining \eqref{eq:qwe} and \eqref{eq:qwe2} we get the statement of the lemma.
\end{proof}
\subsection{Proof of Lemma~\ref{lem:spll}}
\begin{lemma*}
For any rank-$r$ matrix $\gnrmtr\in\Rkm$ and weight vectors $w^\front$ and $w^\back$ with non-negative entries
\begin{align*}
    ||D^{\frac12}(w^\front)\cdot\gnrmtr\cdot D^{\frac12}(w^\back)|| _1
    \leq
    \textstyle\sqrt{r  (\sum_i w^\front(i)) (\sum_j w^\back(j)) }\cdot \norm{}{\gnrmtr}.
\end{align*}
\end{lemma*}
\begin{proof}
\begin{align*}
    \norm1{\gnrmtr} 
    &= \textsum\limits_{i}\textstyle{\sum}_{j}\sqrt{w^\front(i)\cdot w^\back(j)}|\gnrmtr_{ij}| \\ 
    &\overset{\text{(a)}}\leq  \textstyle{\sum}_{i} \sqrt{w^\front(i)} \sqrt{(\textstyle{\sum}_{j} \gnrmtr^2_{ij} )(\textstyle{\sum}_{j}w^\back(j))} \\
    &  =    \sqrt{\textstyle{\sum}_{j}  w^\back(j)}\textstyle{\sum}_{i}|\sqrt{w^\front(i)}| \sqrt{(\textstyle{\sum}_{j} \gnrmtr^2_{ij} )}\\
    & \overset{\text{(b)}}\leq
    \sqrt{\textstyle{\sum}_{j}  w^\back(j)} \sqrt{(\textstyle{\sum}_{i}  w^\front(i) )(\textstyle{\sum}_{i} \textstyle{\sum}_{j} \gnrmtr^2_{ij} )},\\
    & \overset{\text{(c)}}=  \sqrt{\textstyle{\sum}_{j}  w^\back(j)} \sqrt{(\textstyle{\sum}_{i}  w^\front(i) )} \norm F {\gnrmtr_{I\times J}} \\
    & \overset{\text{(d)}}\le \sqrt{r(\textstyle\sum_{j}  w^\back(j))\paren{{\textstyle\sum}_{i\in I}  w^\front(i)}}\cdot \norm{}{\gnrmtr_{I\times J}},
\end{align*}
where (a) and (b) follow from the Cauchy-Schwarz Inequality, (c) from the definition of the Frobenius norm 
\[
\norm F{\gnrmtr}:= \sqrt{\textstyle{\sum}_{i,j} \gnrmtr^2_{ij}},
\]
and (d) as 
\begin{equation*}
\norm F\gnrmtr\le \sqrt{\rank(\gnrmtr)}||\gnrmtr||. \hfill \qedhere
\end{equation*}
\end{proof}

\subsection{Proof of Lemma~\ref{le:subnorm}}
\begin{lemma*}
Let $\gnrmtr\in \Rkm $ and $I_1,I_2,I_3,...,I_t$ be $t$ disjoint subsets of $[k]$ such that $\cup_{j=1}^t  I_j = [k]$. Then
$||\gnrmtr || \leq \sqrt{\sum_{j=1}^t ||\gnrmtr _{I_j}||^2}$.
\end{lemma*}
\begin{proof} Let $v\in \reals^m$ be a  unit vector. Then, 
\begin{equation*}
    \gnrmtr  v = \sum_{j=1}^t  \gnrmtr _{I_j} v = \sum_{j=1}^t  w_j .
\end{equation*}
Here $w_j = \gnrmtr _{I_j} v $. Since $I_j$'s are disjoint, $w_j$'s are orthogonal. 
\begin{equation*}
    ||\gnrmtr  v|| = \sqrt{\sum_{j=1}^t  ||w_j||^2}\leq \sqrt{\sum_{j=1}^t  \max_{||v^*||=1} ||\gnrmtr _{I_j} v^*||^2} = \sqrt{\sum_{j=1}^t   ||\gnrmtr _{I_j} ||^2}.
\end{equation*}
Noting that the above bound holds for any unit vector $v\in \reals^m$ completes the proof.
\end{proof}

\end{document}